\documentclass[11pt]{article}
\usepackage[margin=1in]{geometry}
\usepackage{amsmath, amssymb}
\usepackage[square,numbers]{natbib}  % 方括号引用，常用于工程类期刊
\usepackage{graphicx}
\usepackage{hyperref}

\usepackage{microtype}
\usepackage{graphicx}
\usepackage{subfigure}
\usepackage{booktabs}
\usepackage{amsmath}
\usepackage{amssymb}
\usepackage{amsthm}
\usepackage{algorithm}
\usepackage{algorithmic}
\usepackage{duckuments}
\usepackage{wrapfig}
\usepackage{color}
\usepackage{hyperref}
\usepackage{xspace}         % 自动在命令后添加空格，少数场景使用
\usepackage[table,xcdraw]{xcolor} % 会影响表格配色排版，建议仅
\definecolor{myblue}{HTML}{DDEAF2}
\definecolor{mygreen}{HTML}{2E7D32}

\usepackage[font=small]{caption}
\usepackage{tcolorbox}

\usepackage{bm}
\usepackage{pifont}
\usepackage[table,xcdraw]{xcolor}
\usepackage{multirow}
\usepackage{enumitem}

\usepackage{authblk}  % 用于简单控制作者-单位格式
\usepackage{geometry}
\newcommand{\cmark}{\checkmark}
\newtheorem{theorem}{Theorem}

\newtheorem{proposition}{Proposition}
\newtheorem{lemma}{Lemma}

\newcommand{\ie}{i.e.}

\newcommand{\ourmethod}{Mc-Di\xspace}
\newcommand{\ourmethodo}{Mv-Di\xspace}
\definecolor{myblue}{HTML}{DDEAF2}

\title{Learning to Learn Weight Generation via Local Consistency Diffusion}

\author[1,2]{Yunchuan Guan}
\author[2]{Ke Zhou}
\author[2]{Yu Liu}
\author[1]{Zhiqi Shen}
\author[3]{Jenq-Neng Hwang}
\author[3,4]{Lei Li}

\affil[1]{Nanyang Technological University}
\affil[2]{Huazhong University of Science and Technology}
\affil[3]{University of Washington}
\affil[4]{University of Copenhagen}

\date{3 Feb 2025}  % 可留空，arXiv 会自动加日期

\begin{document}
\maketitle

\begin{abstract}

Diffusion-based algorithms have emerged as promising techniques for weight generation. However, existing solutions are limited by two challenges: generalizability and local target assignment. The former arises from the inherent lack of cross-task transferability in existing single-level optimization methods, limiting the model’s performance on new tasks. The latter lies in existing research modeling only global optimal weights, neglecting the supervision signals in local target weights. Moreover, naively assigning local target weights causes local-global inconsistency. To address these issues, we propose Mc-Di, which integrates the diffusion algorithm with meta-learning for better generalizability. Furthermore, we extend the vanilla diffusion into a local consistency diffusion algorithm. Our theory and experiments demonstrate that it can learn from local targets while maintaining consistency with the global optima. We validate MC-Di's superior accuracy and inference efficiency in tasks that require frequent weight updates, including transfer learning, few-shot learning, domain generalization, and large language model adaptation.

\end{abstract}

\section{Introduction}
\label{sec:intro}
\label{sec:formatting}

\begin{wrapfigure}{r}{0.5\textwidth}
    \centering
    \vspace{-15pt} % 可调节图像与上方文本的间距
\includegraphics[width=0.5\textwidth]{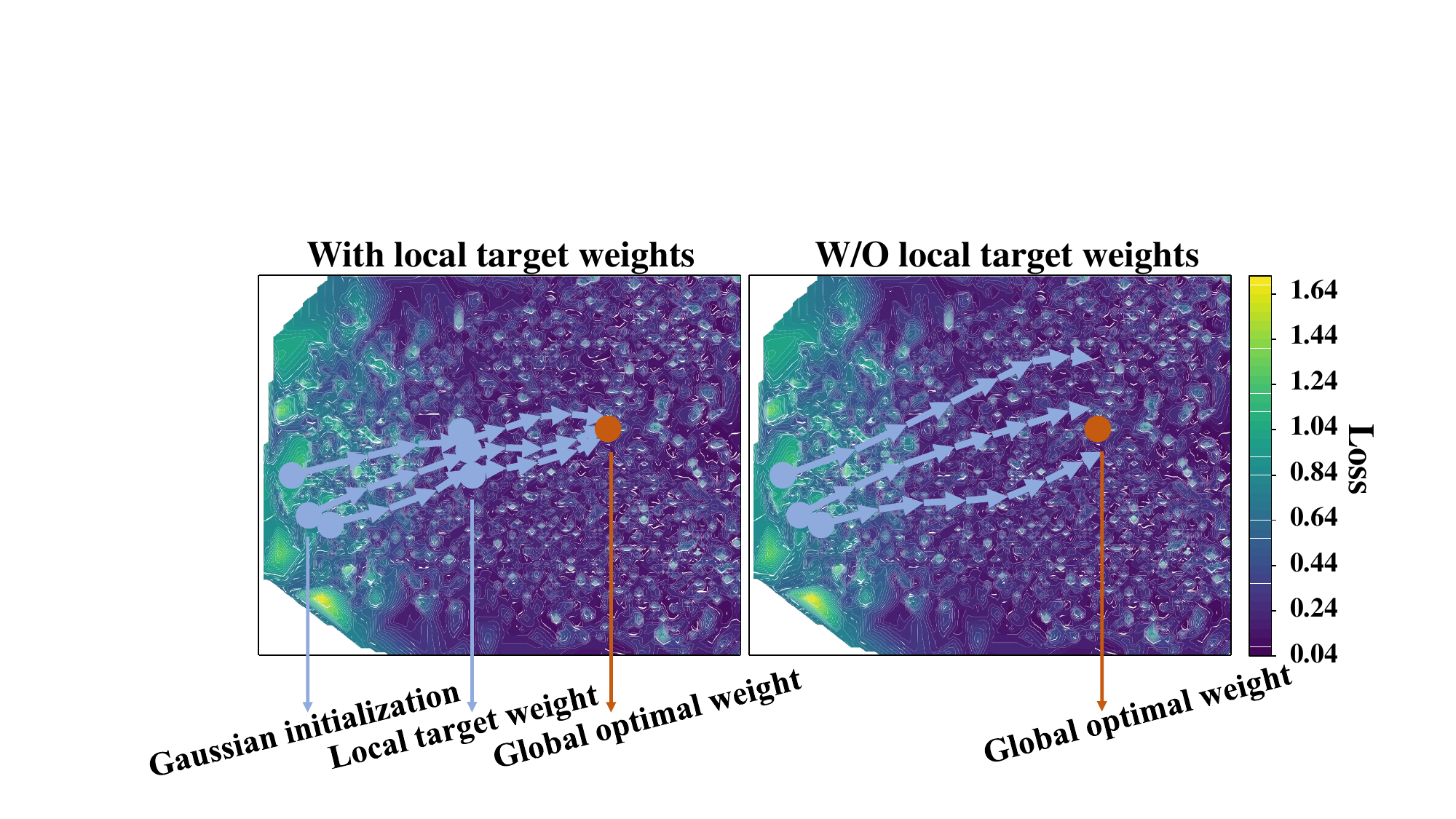}
    \caption{Visualization of inference chains in Omniglot’s 2D weight-reduced space. Darker areas indicate lower task loss. The model trained with local target weights produces accurate and efficient inference chains.}
    \label{fig:landscape_2d}
    % \vspace{-10pt} % 可调节图像与下方文本的间距
\end{wrapfigure}
Diffusion-based models have achieved impressive generative performance in various data modalities, ranging from images to structured tabular~\cite{generative_survey}. Recent works leverage a diffusion model $f^G_{\phi}$ to generate neural network weight $\theta$ for downstream tasks~\cite{OCD,GHN2,GHN3}. This technology enables training and fine-tuning to be performed in a gradient-free manner, offering promising solutions for tasks that require frequent weight updates. Examples include transfer learning, few-shot learning, domain generalization, and LLM adaptation. However, existing methods are limited by two challenges: generalizability and local target assignment.
\vspace{2pt}

\noindent\textbf{Generalizability.}
Current studies leverage models such as Variational Autoencoder (VAE)~\cite{VAE} and Hypernetwork~\cite{hypernetworks} to learn the latent distribution of optimal weights for downstream tasks. OCD~\cite{OCD} and D2NWG~\cite{D2NWG} attempt to use the diffusion model to simulate the weight optimization process. However, these approaches are constrained by single-level optimization frameworks and demonstrate limited capability for knowledge transfer between tasks, This limitation subsequently impacts their generalizability for new tasks~\cite{survey_neural, Bilevel_Optimization_generalization, meta_learning_rate}. 
\vspace{2pt}

\noindent\textbf{Local target assignment.}
Existing methods only model global optima $\theta_{M}$, overlooking the supervision signals in local target weights. These weights are acquired during optimization toward the global optima and contain the policy-specific details of the optimizer. Figure~\ref{fig:landscape_2d} shows that ignoring local targets produces inefficient and low-accuracy inference chains. Moreover, as shown in Figure~\ref{fig:consistency}, naively assigning local target weights creates inconsistency between local and global targets. Consequently, assigning local targets while maintaining consistency remains an open and non-trivial challenge.

\begin{figure*}[t]
    \centering
    \includegraphics[width=0.95\linewidth]{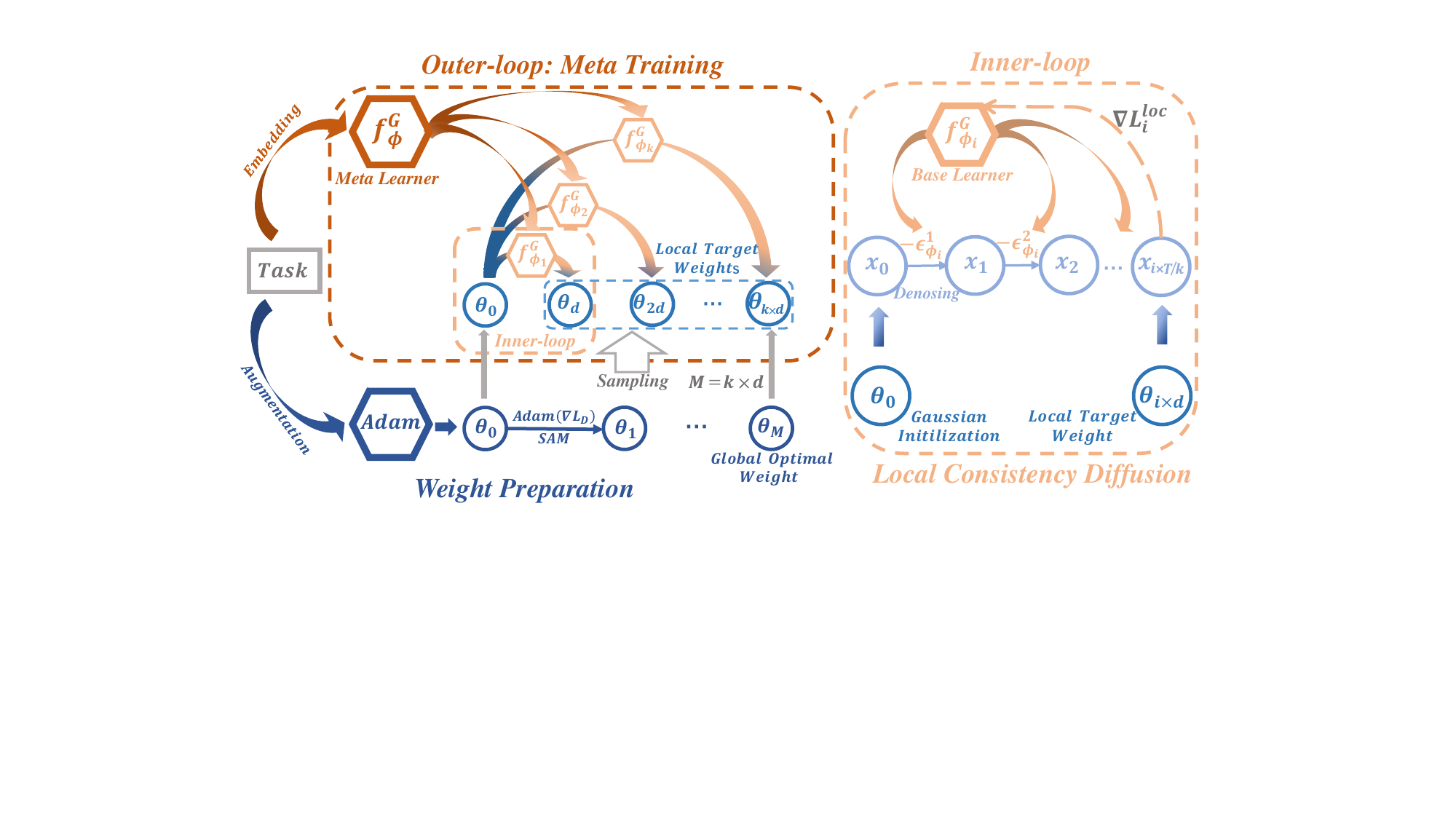}
    \caption{Workflow of \ourmethod. In the weight-preparation stage, a real-world optimizer (Adam) is used to optimize the downstream task. We collect the optimization trajectory, i.e., $\{\theta_i\}_{i=0}^M$, and sample from them to obtain local target weights $\{\theta_{i\times d}\}_{i=1}^k$. In the meta-training stage, a meta-learner $f^G_{\phi}$ assigns a base learner $f^G_{\phi_i}$ to each local target weight $\theta_{i\times d}$. Within each inner-loop, the base learner models local targets using local consistency diffusion.}
    \vspace{-1em}
    \label{fig:overview}
\end{figure*} 

\vspace{5pt}

In this paper, (1) we propose \textit{\textbf{M}}eta-learning weight generation via local \textit{\textbf{c}}onsistency \textit{\textbf{Di}}ffusion, \ie, \ourmethod. Figure~\ref{fig:overview} shows \ourmethod's workflow, which consists of weight preparation and meta-training stages. In the weight preparation stage, local target weights $\{\theta_{i\times d}\}_{i=1}^k$ are collected by sampling the optimization trajectory at uniform intervals of $d$. In the meta-training stage, we combine bi-level optimization and diffusion algorithms to model local targets for better generalizability. (2) We extend vanilla diffusion to Local Consistency Diffusion. It models local target while maintaining global consistency. Figure~\ref{fig:landscape_2d} shows how local consistency diffusion leverages local targets to guide the inference process. The above characteristics enable \ourmethod to achieve superior generative accuracy and improved inference efficiency. (3) We analyze the convergence properties of the weight generation paradigm and introduce improvements. We decouple functional components, \ie, data augmentation~\cite{survey_neural} for robustness and Sharpness-Aware Minimization (SAM)~\cite{SAM} for convergence efficiency, from the meta-training stage to the weight preparation stage. This approach improves convergence efficiency without additional time overhead. 

Our contributions can be summarized as follows:
\begin{itemize}
    \item We propose \ourmethod that combines meta-learning and diffusion algorithms for better generalizability.
    \item We propose local consistency diffusion to model local target weights. It maintains consistency between local and global targets, thus improving generation quality.
    \item We study the convergence of the weight generation paradigm and introduce SAM to improve efficiency without extra time overhead.
\end{itemize}

\section{Methodology}\label{sec:method}

The proposed \ourmethod consists of two stages: weight preparation and meta-training. It combines bi-level optimization meta-learning with a diffusion algorithm, enabling better generalizability.
\\
\par
\subsection{Weight Preparation}\label{sec:weight preparation}
As shown below in Figure~\ref{fig:overview}, the weight preparation stage aims to collect and sample local target weights. These weights are acquired during iterations toward the global optima $\theta_M$ and contain the policy-specific details of the optimizer, where $M$ represents the downstream training epoch.

Specifically, we use task data $T_i$ and a real-world optimizer, \ie, Adam~\cite{Adam}, to optimize the downstream neural network and record the corresponding optimization trajectory $\{\theta_0, \theta_1, ... ,\theta_M\}$. Note that $\theta_0$ is a Gaussian-initialized weight. In our implementation, we uniformly sample from the optimization trajectory. Consequently, the sampled local target weights are $\{\theta_{d}, \theta_{2d}, ...,\theta_{k\times d}\}$, where $d=M/k$ and $k$ is the segment number.

Benefiting from the two-stage weight generation paradigm, we can advance functional components from the inner-loop to the weight preparation stage. As shown in Figure~\ref{fig:overview}, we add Sharpness Aware Minimization (SAM) for convergence efficiency and data augmentation for robustness into the weight preparation stage. Meanwhile, since the meta-training stage can start synchronously once a few weights are collected, adding functional components during weight preparation doesn't incur additional time overhead. The specific process of the weight preparation stage is detailed in Algorithm~\ref{alg:SAM}.
\\
\par
\subsection{Meta-Training}
In the meta-training stage, we use a bi-level optimization algorithm, \ie, REPTILE~\cite{REPTILE}, as the framework to ensure efficiency. As shown in the orange region of Figure~\ref{fig:overview}, the meta-learner $f^G_{\phi}$ assigns a base learner $f^G_{\phi_i}$ to each local target $\theta_{i\times d}$.\footnote{The terms “diffusion model”, “meta-learner” $f^G_{\phi}$, and “denoiser” $\epsilon_{\phi}$, are used interchangeably in this paper.} Within each inner-loop, the base learner models the generation process of $\theta_0 \rightarrow \theta_{i\times d}$ using local consistency diffusion.
To enable learning across multiple tasks and enhance meta-learning, the diffusion model is conditioned on the task embedding $Emb_{T_i}$, which is calculated by a ResNet101 backbone used by~\cite{ICIS}. Considering $\theta_{i\times d}$ as the generation target $x_T$, a naive inner-loop objective, \ie, vanilla diffusion loss $L_i$, can be written as
\begin{align}\label{eq:inconsistency_loss}
    L_i&=\underset{t\in[0,T)}{E}||\epsilon_{\phi_i}(x_t, t, Emb_{T_i})-\epsilon||^2,\\
    x_t&=\sqrt{\bar{\alpha}_t} \theta_{i\times d} + \sqrt{1 - \bar{\alpha}_t}\epsilon, \notag,
\end{align}

As shown in Figure~\ref{fig:consistency}, vanilla diffusion loss $L_i$ disrupts local-global consistency. In Section~\ref{sec:Local Consistency Diffusion}, we propose a local consistency loss $L^{loc}_i$ for this issue, and here we just denote the inner-loop objective by $L^{loc}_i$.

According to REPTILE, the inner-loop update equation can be written as
$$
    \phi_i=\phi_i-\eta\nabla_{\phi_i} L_i^{loc},
$$
where $t$ is the timestamp, $\eta$ is the inner-loop learning rate, $\bar{\alpha}_t=\prod_{j=t}^{T-1} \alpha_j$, and $\{\alpha_j\}^{T}_{j=0}$ is an increasing schedule sequence. Then, the outer-loop process can be written as
$$
\phi=\phi+\frac{\zeta}{B}\sum_{i=1}^B({\phi_i}-\phi),
$$
where $\zeta$ is the outer-loop learning rate and $B$ is the meta-batch size. The overall meta-learning process is shown in Algorithm~\ref{alg:training}. It may raise concerns about computational complexity, as \ourmethod introduces additional local targets for learning. We detail this issue in Section~\ref{sec:Trajectory Selection}.
% \vspace{-5pt}

\begin{figure}[t]
\centering
\begin{minipage}[t]{0.48\linewidth}
    \begin{algorithm}[H]
    \small
    \caption{Weight Preparation}
    \label{alg:SAM}
    \begin{algorithmic}[1]
        \REQUIRE Downstream task weights $\theta_0$, downstream task loss $L^D$, perturbation $\rho$, task distribution $\mathcal{T}$.
        \FOR{$T_j \sim \mathcal{T}$}
            \FOR{$t$ in range($M$)}
                \STATE Add noise and rotated data (\ie, data augmentation)
                \STATE $\epsilon \gets \rho \frac{\nabla L^D(T_j;\theta_t)}{\|\nabla L^D(T_j;\theta_t)\|}$~~(SAM)
                \STATE $\theta_{t+1} \gets Adam(\theta_t, \nabla_{\theta} L^D(T_j;\theta_t + \epsilon))$
                \STATE Append $\theta_t$ in $W_j^{of}$
            \ENDFOR
        \STATE Uniformly sample local target weights $W^{loc}_j=\{\theta^j_{i\times d}\}_{i=1}^k$ from $W_j^{of}$.
        \ENDFOR
    \STATE \textbf{Return} $\{W^{\!loc}_j\}_{j=1}^n$
    \end{algorithmic}
    \end{algorithm}
\end{minipage}
\hfill
\begin{minipage}[t]{0.48\linewidth}
    \begin{algorithm}[H]
    \small
    \caption{Meta Training (Batch version)}
    \label{alg:training}
    \begin{algorithmic}[1]
        \REQUIRE Denoiser weight $\phi$, inner-loop step $K$, learning rate $\eta$, meta-learning rate $\zeta$, local target weights set $\{W^{\!loc}_j\}_{j=1}^n$, meta-batch size $B$.
        \WHILE{not converged}
            \FOR{repeated $B$ times}
                \STATE Randomly sample $j \in [0,n), i\in (0,k]$ to determined a local target weight $\theta^j_{i\times d}$
                \STATE $\phi_i \gets \phi$
                \FOR{repeated $K$ times}
                    \STATE $\phi_i \gets \phi_i - \eta\nabla_{\phi_i}L_i^{loc}(\theta^j_{i\times d},\phi_i)$~~(Eq.~\ref{eq:consistency_loss})
                \ENDFOR
                \STATE $\Delta_{\phi_i} \gets \phi_i - \phi$
            \ENDFOR
            \STATE $\phi \gets \phi + \frac{\zeta}{B}\sum_{i=1}^B \Delta_{\phi_i}$
        \ENDWHILE
    \end{algorithmic}
    \end{algorithm}
\end{minipage}
\end{figure}

\subsection{Downstream Task Evaluation}
\vspace{-5pt}
The inference process of \ourmethod aligns with the vanilla diffusion algorithm. It can be written as
\begin{equation}\label{eq:DDPM_inference}
x_{t+1} = \frac{1}{\sqrt{\bar{\alpha}_{t+1}}} \left( x_t - \sqrt{1 - \bar{\alpha}_{t+1}} \epsilon_\phi(x_t, t) \right).
\end{equation}
When $i = k$, the local target weight $\theta_{k\times d}$ coincides with the global optimal weight $\theta_M$. This implies that the well-trained diffusion model $f_{\phi^*}^G$ can recover the optimal weight $\theta_{M}$ from Gaussian-initialized weights through iterative inference.\footnote{We will omit the condition state $Emb_{T_i}$ in the following section for brevity.}

\section{Local Consistency Diffusion and Analysis}\label{sec:trajectory overall}
In this section, we introduce the proposed local consistency diffusion algorithm and analyze its computational complexity, principle, and convergence.

\subsection{Local Consistency Diffusion}\label{sec:Local Consistency Diffusion}

\begin{wrapfigure}{r}{0.3\textwidth}
    \centering
    \vspace{-60pt}
    \includegraphics[width=0.3\textwidth]{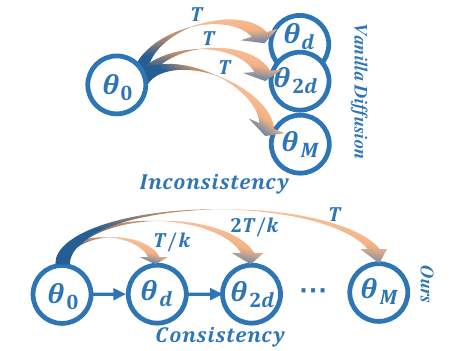}
    % \vspace{-20pt}
    \caption{Naively assigning local targets creates inconsistency.}
    \label{fig:consistency}
    % \vspace{-40pt}
\end{wrapfigure}

Directly using vanilla diffusion loss in Equation~\ref{eq:inconsistency_loss} for meta-training creates inconsistency between local and global targets. The top of Figure~\ref{fig:consistency} shows that local targets $\theta_{i\times d}$ and the global optima $\theta_M$ are equally treated as the generation target $x_T$ after $T$ denoising steps.\footnote{Note that the subscript order of $x_t$ here is reversed compared to the standard diffusion setup. This is to align with the subscript order of the target weights $\theta_i$.} This naive approach creates inconsistency between local and global targets, leading the introduced local target to negatively affect overall performance. To address this issue, the principle of local consistency diffusion can be described below.

\begin{tcolorbox}[colback=myblue, colframe=gray!80!black]
Local consistency diffusion aims to generate a sequence of weights starting from Gaussian noise, where $\theta_d$, $\theta_{2d}$, ..., $\theta_{M=k\times d}$ can be reached at evenly spaced intervals of $T/k$ steps.
\end{tcolorbox}

Based on the above principle, we define the local consistency loss $L^{loc}$ below.

\begin{theorem}\label{the:L_k}
Given the number of diffusion steps $T$, an increasing schedule $\{\alpha_0,...,\alpha_T\}$, local target weights $\{\theta_d,...,\theta_{k\times d}\}$, and let the inference process align with the vanilla diffusion algorithm, \ie, 
\begin{equation*}
x_{t+1} = \frac{1}{\sqrt{\bar{\alpha}_{t+1}}} \left( x_t - \sqrt{1 - \bar{\alpha}_{t+1}} \epsilon_\phi(x_t, t) \right).
\end{equation*}

Then, the denoiser $\epsilon_{\phi}$ can recover the target sequence $\{\theta_d,...,\theta_{M=k\times d}\}$ from standard Gaussian noise $x_0$ with evenly $T/k$ steps intervals, when $\epsilon_{\phi}$ is trained by local consistency loss $L^{loc}$ as follows:
\begin{align}\label{eq:consistency_loss}
&L^{loc}=\underset{i\in(0,k]}{E}~L^{loc}_i,\notag\\
&L_i^{loc}=\underset{t\in[0,i\times T/k)}{E}||\sqrt{1-\bar{\alpha}_t^i}\epsilon_{\phi}(x_t, t)-\sqrt{1-\bar{\alpha}_t}\epsilon||^2,\\
&x_t=\sqrt{\bar{\alpha}^i_t} \theta_{i\times d} + \sqrt{1 - \bar{\alpha}^i_t}\epsilon, \notag
\end{align}
where $\bar{\alpha}_t=\prod_{j=t}^{T-1} \alpha_j$,  $\bar{\alpha}^i_t=\prod_{j=t}^{i\times T/k-1} \alpha_j$, and $\epsilon$ denotes standard Gaussian noise.
\end{theorem}

The proof is detailed in Appendix~\ref{sec:appendix_1}. In fact, when $k=1$, local consistency diffusion is equivalent to the vanilla diffusion algorithm. At this point, the denoiser solely models the global optimal weight $\theta_M$. We prove this claim in Appendix~\ref{sec:appendix_propo}. We refer to the model corresponding to this case as \ourmethodo. The right side of Figure~\ref{fig:landscape_2d} depicts the inference chains of \ourmethodo. Moreover, existing methods~\cite{OCD,MetaDiff,D2NWG} reduce to the $k=1$ case, so \ourmethod can be regarded as their generalization.

\subsection{Analysis}\label{sec:analysis} In this section, we discuss issues concerning computational complexity, principles, and convergence.
\begin{itemize}
    \item How does the segment number $k$ affect computational complexity?
    \item How does local consistency diffusion enable better efficiency and accuracy?
    \item The convergence properties of the weight generation paradigm and how to improve them.
\end{itemize}

\subsubsection{Computational complexity}\label{sec:Trajectory Selection}
In practice, local consistency diffusion achieves lower computational cost compared to the vanilla diffusion algorithm. This claim contradicts the intuition from Equation~\ref{eq:consistency_loss}, where local consistency loss has a computational complexity of $O(k\times T/2)$, compared to $O(T)$ for the vanilla diffusion loss. However, experimental results in Figure~\ref{fig:loss_trajectory} and analysis of Figure~\ref{fig:loss_hie} support our claim.

Under the constraint of three GPU hours, Figure~\ref{fig:loss_trajectory} shows the relationship between the segment number $k$ and reconstruction MSE (Mean Square Error) between the final predicted weight and the ground truth weight. We perform a case study on Omniglot~\cite{omniglot} and Mini-Imagenet~\cite{miniImagenet} datasets for 5-way tasks. When $k=1$, our method degenerates into vanilla diffusion. When $k > 1$, our method improves generation accuracy by introducing additional local target $\{\theta_{i\times d}\}_{i=1}^{k-1}$. Note that this result is obtained under the same 3 GPU-hours constraint, highlighting the lower practical computational cost of our method. Considering the stable trade-off across multiple datasets, we set $k=3$ as the default in subsequent experiments. 

\subsubsection{Efficiency and Accuracy Improvement}\label{sec:discuss acceleration}

The reduced computational cost of our method can be loosely explained by its division of a search space with radius $T$ into $m$ smaller ones with radius $T/k$. Figure~\ref{fig:loss_hie} supports the above claim. When $k=3$, we record three optimization objectives, \ie, $L^{loc}_1$, $L^{loc}_2$, and $L^{loc}_3$, corresponding to two local target weights $\theta_{M/3},\theta_{2M/3}$ and a global optimal weight $\theta_M$, on Omniglot 5-way 1-shot tasks. It can be found that $L^{loc}_1$ converges first. This is because $\theta_{M/3}$ is closest to the Gaussian-initialized weight ${\theta_0}$ and requires the fewest training epochs. Thus, the remaining learning target is to generate ${\theta_M}$ starting from $\theta_{M/3}$. By recursively applying this operation, a $T$-steps diffusion problem breaks down into $k$ shorter ones. As a result, our method improves practical computational complexity.

\begin{figure}[t]
    \centering
    \begin{minipage}[t]{0.24\linewidth}
        \centering
        \includegraphics[width=\linewidth]{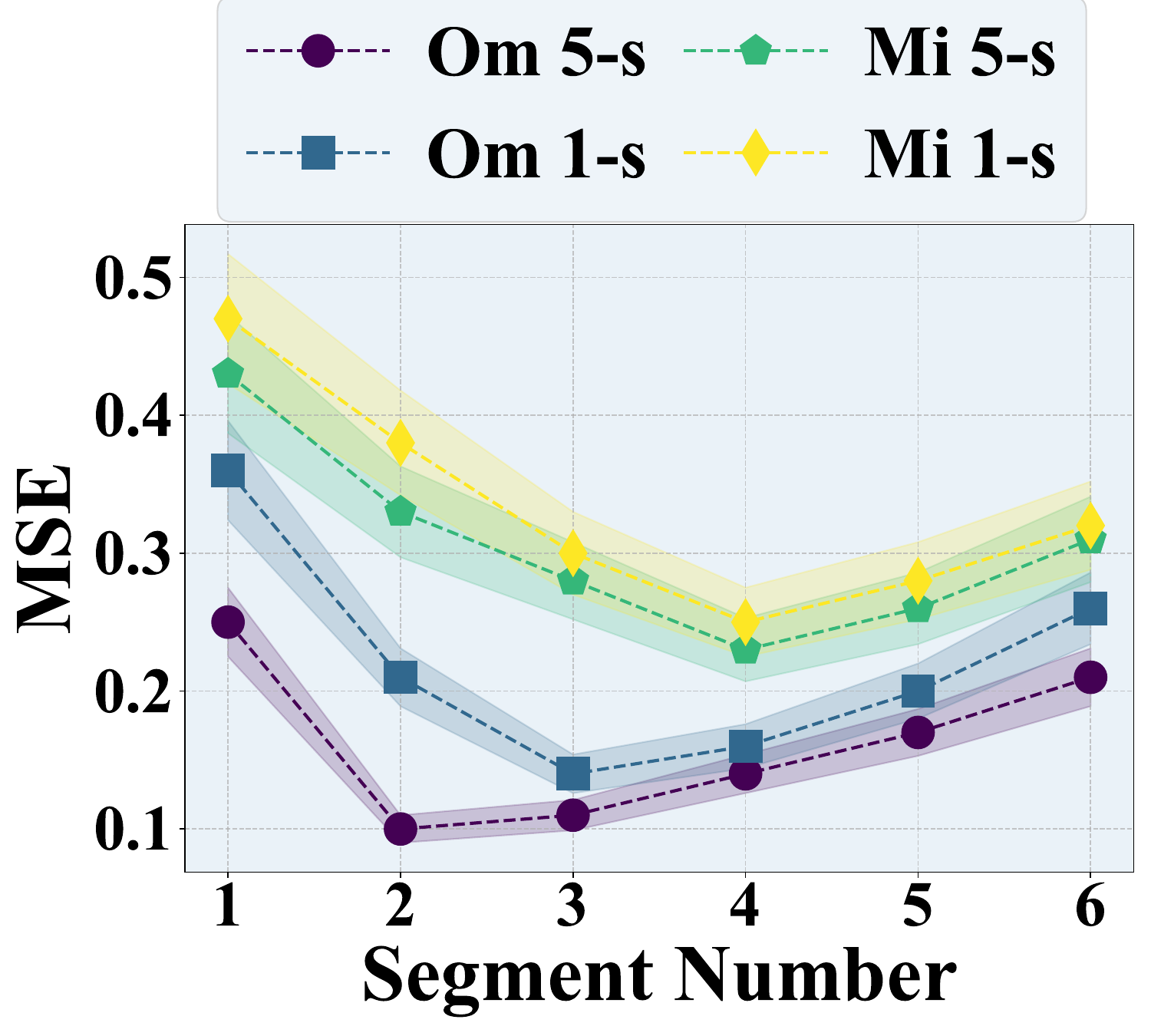}
        \caption{Segment Number vs. MSE trade-off on 5-way 1-shot and 5-shot tasks.}
        \label{fig:loss_trajectory}
    \end{minipage}
    \hspace{0.01\linewidth}
    \begin{minipage}[t]{0.24\linewidth}
        \centering
        \includegraphics[width=\linewidth]{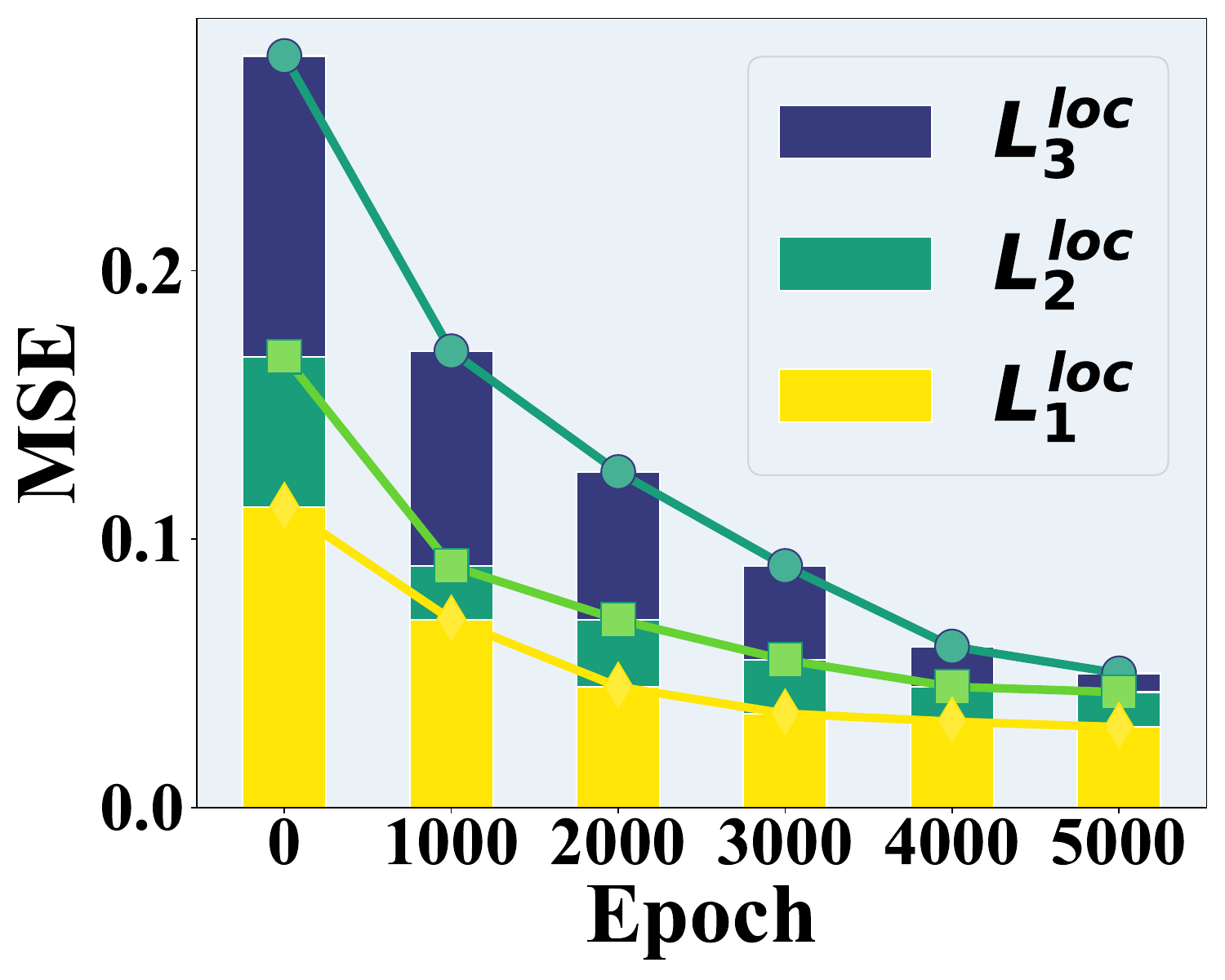}
        \caption{Convergence speed of each $L^{loc}_i$ on Omniglot 5-way 1-shot task.}
        \label{fig:loss_hie}
    \end{minipage}
    \hspace{0.01\linewidth}
    \begin{minipage}[t]{0.47\linewidth}
        \centering
        \includegraphics[width=\linewidth]{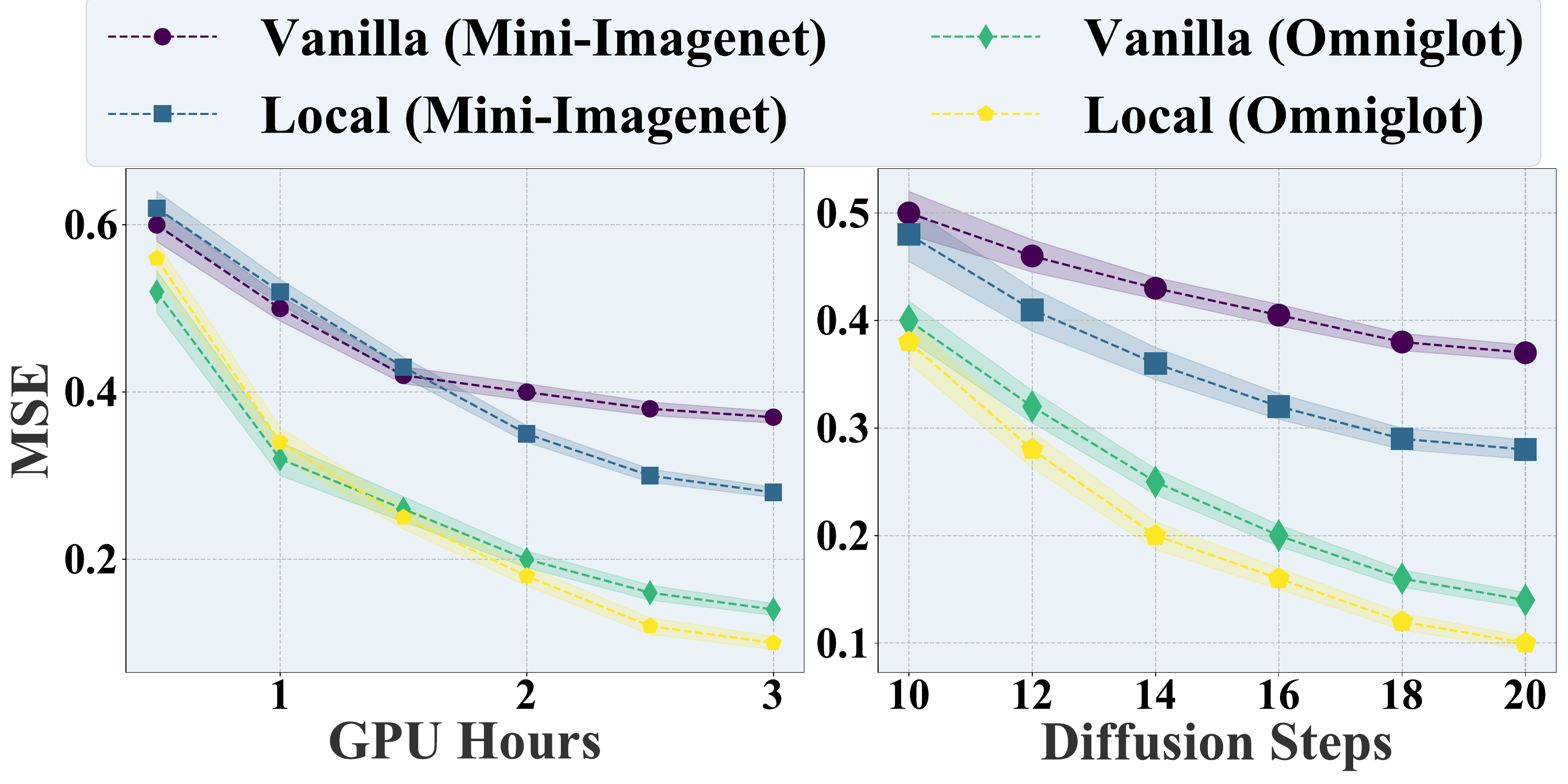}
        \caption{Comparison between vanilla and local consistency diffusion on Mini-ImageNet and Omniglot (5-way 1-shot). Left: GPU hours vs. MSE. Right: Diffusion steps vs. MSE.}
        \label{fig:trade-off}
    \end{minipage}
\end{figure}

Figure~\ref{fig:trade-off} compares the trade-off curves between the vanilla diffusion and the local consistency diffusion on Omniglot and Mini-Imagenet 5-way 1-shot tasks. The left figure shows GPU-Hours vs. MSE trade-off, and the right one is Diffusion Steps vs. MSE trade-off. As training progresses, our method increasingly outperforms the vanilla diffusion. The reason is that, in the early stages of training, local consistency diffusion primarily focuses on the learning of prerequisite objectives, such as $L_1$ and $L_2$. Once the learning of these prerequisite objectives is completed, our method converges faster to the optimal weights, \ie, $\theta_M$. Benefiting from this, as shown on the right side of Figure~\ref{fig:trade-off}, local consistency diffusion can tolerate fewer diffusion steps for the same MSE, thereby enhancing both efficiency and accuracy.

\subsubsection{Convergence Analysis and Improvement}\label{sec:SAM}
The paradigm of learning to generate weights involves two independent loss terms, \ie, downstream task loss $L^D$ and the local consistency loss $L^{loc}$. This makes ensuring the convergence of the paradigm a non-trivial issue. The following convergence analysis holds for all weight generation methods.
% \begin{assumption}\label{assumption}
% \noindent
% \begin{enumerate}
%     \item \textbf{Reconstruction error upper bound.}
%     The reconstruction error produced by the generative model is bounded by $c$.
%     \item \textbf{Loss function upper bound.}
%     The downstream task loss $L_d(\cdot) \leq \psi$.
%     \item \textbf{l-smoothness.}
%     There exists a constant $l$ such that for all weights \( \theta, \theta' \in \mathbb{R}^n \)
%     \[
%     \|\nabla L_d(\theta) - \nabla L_d(\theta')\| \leq l \|\theta - \theta'\|
%     \]
%     \item \textbf{$\mu$-strong convexity.}
%     There exists a constant \( \mu > 0 \) such that for all \( \theta,\theta' \in \mathbb{R}^n \),
%     \[
%     L_d(\theta') \geq L_d(\theta) + \nabla L_d(\theta)^\top (\theta' - \theta) + \frac{\mu}{2} \| \theta' - \theta \|^2.
%     \]
%     \item \textbf{Hessian matrix upper bound.}
%     The Hessian matrix eigenvalue around the neighborhood of optimal weight $\theta^*$ is bounded by $\lambda$.
% \end{enumerate}
% \end{assumption}

% \begin{theorem}\label{theorem:emperi error}
% When Assumption~\ref{assumption} holds, Lt-Di's cumulative empirical error can be bound by:
% $$
%     L_d(\hat{\theta})-L_d(\theta^*) \leq \frac{\lambda}{2} \left[c+\frac{2\psi}{\mu}\left(1 - \frac{\mu}{l}\right)^{epoch}\right],
% $$
% where $epoch$ is the number of update steps used in the weight preparation stage, and $\hat{\theta}$ is the weight predicted by the generative model.
% \end{theorem}

\begin{theorem}\label{theorem:emperi error}
Assume that the reconstruction error of the generative model is bounded by \( c \), the downstream loss is bounded by $\psi$, and the loss function is both \( l \)-smooth and \( \mu \)-strongly convex, with the eigenvalues of the Hessian matrix around the optimum \( \theta_* \) bounded by \( \lambda \). Then, the cumulative empirical error of the weight generation paradigm can be bounded as follows:
\begin{equation}
    L^D(\hat{\theta})-L^D(\theta^*) \leq \frac{\lambda}{2} \left[c+\frac{2\psi}{\mu}\left(1 - \frac{\mu}{l}\right)^{M}\right],
\end{equation}
where $\hat{\theta}$ is the weight predicted by the generative model.
\end{theorem}

The proof, provided in Appendix~\ref{sec:appendix_2}, relies on the triangle inequality to decompose the accumulated error into weight preparation error and reconstruction error. Theorem~\ref{theorem:emperi error} shows that, such a weight generation paradigm affects the overall convergence in only a linear manner, \ie, $c+\frac{2\psi}{\mu}\left(1 - \frac{\mu}{l}\right)^{M}$. Furthermore, this convergence upper bound can be effectively improved by reducing the maximum eigenvalue $\lambda$. 

We introduce Sharpness-Aware Minimization (SAM)~\cite{SAM} into the weight preparation stage, to penalize $\lambda$ by constraining the curvature near global optima. Section~\ref{sec:weight preparation} mentioned that the additional SAM component doesn't incur additional time overhead. Figure~\ref{fig:components_ablation} shows the effect of SAM on convergence efficiency. The process of SAM is shown in Algorithm~\ref{alg:SAM}. Note that although $\mu$-convex is a strong assumption, it does not limit the practicality of our analysis and using SAM. This is because the overall structure of the accumulated error $\lambda(a+b)$ remains unaffected by this assumption.

% \begin{align}\label{eq:total_loss}
%     L&=\sum_{i=0}^{m} \gamma_i L_i\notag\\
%     &=\sum_{i=0}^{m} \gamma_i \sum_{t=k+1}^T||\sqrt{1-\bar{\alpha}_t^k}\epsilon_{\phi}(x_t, t-k)-\sqrt{1-\bar{\alpha}_t}\epsilon_k||^2\cdot \mathbf{1}_{\{t \in (k,r(i+1)]\}}
% \end{align}
% where $k=r(i)$, $x_t=\sqrt{\bar{\alpha}^k_t} \mathbf{x}_k + \sqrt{1 - \bar{\alpha}^k_t}\epsilon_k$, $\bar{\alpha}_t=\prod_{j=1}^t \alpha_j$, $\bar{\alpha}^k_t=\prod_{j=k+1}^t \alpha_j$, $\gamma_i$ is the coefficient of $L_i$, and $\epsilon_k$ is a standard Gaussian noise.
\section{Experiment}\label{sec:experiment}
Our experimental platform includes two A100 GPUs, one Intel Xeon Gold 6348 processor, and 512 GB of DDR4 memory. For all experiment results, we report the mean and standard deviation over 5 independent repeated experiments. In the following section, we present the basic experimental results and setups, with more details provided in Appendix~\ref{sec:appendix_4}.
\subsection{Ablation Study}\label{sec:ablation}
\subsubsection{Main Components}
The advantages of \ourmethod stem from these main components. 
\begin{itemize}[align=left]
  \item C1: Combining meta-learning with diffusion-based weight-generation.
  \item C2: Introducing local target weights to guide the inference chain.
  \item C3: Using local consistency loss to maintain local–global consistency.
\end{itemize}

Note that C2 builds on C1, and C3 builds on C2. Therefore, we conduct an incremental ablation study. As shown in Table~\ref{tab:ablation_main}, we validate the effectiveness of each component on Omniglot and Mini-ImageNet 5-way 1-shot tasks. When none of the components are used, \ourmethod degrades to the REPTILE (meta-learning only) or OCD~\cite{OCD} (diffusion only). When only C1 is used, \ourmethod degrades to \ourmethodo. When only using C1 and C2, we refer to this method as Tw-Di.

The comparison between REPTILE and \ourmethodo shows the effectiveness of diffusion-based methods. OCD employs a single-layer optimization diffusion, and its comparison with \ourmethodo emphasizes the necessity of using meta-learning's bi-level optimization. When local target weights are introduced without maintaining local–global consistency, Tw-Di underperforms \ourmethodo, highlighting the importance of consistency. When introducing local consistency loss, \ourmethod achieves the highest accuracy. 

\begin{figure}[t] % 使用 float 提供的 [H] 参数强制置顶
\centering
\begin{minipage}[h]{0.48\linewidth}
    \centering
    \captionsetup{type=table}
    \resizebox{\linewidth}{!}{
    \begin{tabular}{lcccccc}
        \toprule
        & \textbf{C1} & \textbf{C2} & \textbf{C3} & \textbf{Omniglot} & \textbf{Mini-Imagenet} \\
        \midrule
        REPTILE & & & & 95.39 ± 0.09 & 47.07 ± 0.26 \\
        OCD & & & & 95.04 ± 0.18 & 59.76 ± 0.27 \\
        \ourmethodo & \cmark & & & 96.65 ± 0.19 & 62.53 ± 0.37 \\
        Tw-Di & \cmark & \cmark & & 94.28 ± 0.31 & 49.72 ± 0.23 \\
        \rowcolor{myblue} \ourmethod & \cmark & \cmark & \cmark & 97.34 ± 0.19 & 64.87 ± 0.32 \\
        \bottomrule
    \end{tabular}}
    \captionof{table}{Ablation of main components on Omniglot and Mini-Imagenet datasets. We record the accuracy of each variant on 5-way 1-shot tasks.}
    \label{tab:ablation_main}
\end{minipage}
\hfill
\begin{minipage}[h]{0.48\linewidth}
    \centering
    \vspace{-32pt}
    \includegraphics[width=\linewidth]{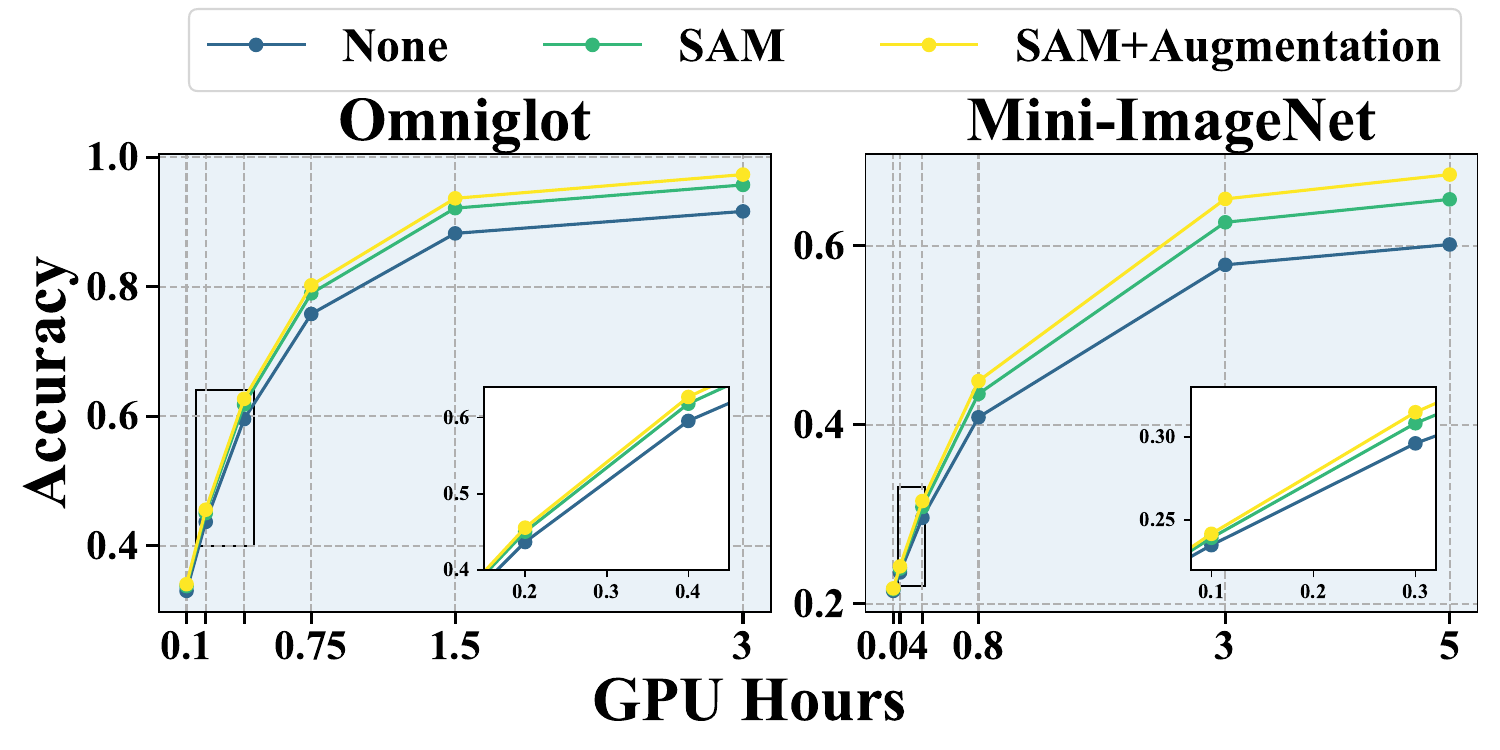}
    \vspace{-20pt}
    \captionsetup{type=figure}
    \captionof{figure}{Ablation of functional components. Accuracy vs. GPU Hours trade-off curves on Omniglot 5-way 1-shot tasks.}
    \label{fig:components_ablation}
\end{minipage}
\end{figure}

\subsubsection{Functional components}\label{sec:ablation_functional}
Following the setup given by~\citep{MAML}, we conducted a case study to evaluate the accuracy improvement and overhead burden brought by each functional component. We incrementally added SAM and data augmentation components to \ourmethod's data preparation stage. In the Omniglot and Mini-Imagenet datasets, we construct 5-way 1-shot tasks and record the model's accuracy on the meta-test set at each period of training. Figure~\ref{fig:components_ablation} shows the GPU Hours vs. Accuracy trade-off curve. The three curves in the figure exhibit identical convergence rates and progressively increasing accuracy. These results demonstrate that adding functional components can improve the model's performance without additional time overhead at \textbf{\textit{any stage of training}}.

\subsection{Comparison Experiments}\label{sec:main_experiments}

\subsubsection{Transfer Learning}
\noindent\textbf{Task.} In this task, we train and evaluate models on disjoint pre-training and evaluation datasets. During evaluation, we do not use any labeled data to adjust the models. We evaluated the generalizability of the models using both accuracy and average per-sample evaluation latency.
\par
\noindent\textbf{Dataset.} We partitioned ImageNet-1k~\cite{imagenet} into 20k subsets of 50 classes each with 50 images per class per task for meta-training. The evaluation datasets are CIFAR-10, CIFAR-100~\cite{CIFAR1O_100}, STL-10~\cite{STL10}, Aircraft~\cite{aifcraft}, and Pets~\cite{pets}.
\par
\noindent\textbf{Baselines.}
We benchmark against ICIS~\cite{ICIS}, GHN3~\cite{GHN3}, and D2NWG~\cite{D2NWG}. ICIS directly learns from downstream task samples, while the others learn to generate weights. For all the aforementioned models, the downstream network uses ResNet12~\cite{resnet} with a linear probe for classification. For the diffusion model $f^G_{\phi}$, \ourmethod employs the same U-Net architecture given by Meta-Diff~\cite{MetaDiff}. The task embedding $Emb_{T_i}$ is calculated by a ResNet101 backbone used by ICIS. In the weight preparation stage, the real-world optimizer (Adam) uses a fixed learning rate of 0.005 and an automatic early-stopping strategy~\cite{early-stop} to determine the downstream task training epoch $M$. In the meta-training stage, we set the learning rate $\eta$, meta-learning rate $\zeta$, inner-steps, and training epochs to 0.005, 0.001, 3, and 6000, respectively. The segment number $k$ is set to 3, and the diffusion step $T$ is set to 20. \textbf{We maintain this setup across all experiments in this paper.}
\par
\noindent\textbf{Results.}
Table~\ref{tab:zero_shot_performance} shows that \ourmethod achieves the highest accuracy on five out of six tasks while reducing average latency. Compared to the second-best method (except for \ourmethodo) on each task, it yields an average accuracy improvement of 2.42\%. On CIFAR-100, its accuracy is only 0.81\% lower than the best-performing method. Compared to the fastest existing method (\ie, D2NWG), \ourmethod reduces inference latency by 1.9$\times$. As discussed in Section~\ref{sec:ablation}, \ourmethod also outperforms its simplified variant \ourmethodo, demonstrating the effectiveness of introducing local target weights.

\begin{table*}[t]
\centering
\resizebox{\linewidth}{!}{
\begin{tabular}{lllllll}
\hline
\textbf{Method} & \textbf{CIFAR-10} & \textbf{CIFAR-100} & \textbf{STL-10} & \textbf{Aircraft} & \textbf{Pets} & \textbf{Latency (ms)} \\ \hline
ICIS~\cite{ICIS} & 61.75 ± 0.31 & 47.66 ± 0.24 & 80.59 ± 0.12 & 26.42 ± 1.56 & 28.71 ± 1.60 & 9.2 \\
GHN3~\cite{GHN3} & 51.80 ± 0.42 & 11.90 ± 0.45 & 75.37 ± 0.19 & 23.19 ± 1.38 & 27.16 ± 1.08 & 14.5 \\
D2NWG~\cite{D2NWG} & 60.42 ± 0.75  & \textbf{51.50 ± 0.25} & 82.42 ± 0.04 & 27.70 ± 3.24 & 32.17 ± 6.30 & 6.7 \\
\ourmethodo(ours) & 61.14 ± 0.32 & 49.62 ± 0.36 & 81.43 ± 0.38 & 29.37 ± 1.22 & 30.28 ± 1.29 & 4.7 \\
\rowcolor{myblue} \ourmethod(ours) & 
\textbf{63.57 ± 0.28}~~\textcolor{red}{$\uparrow$1.82} & 
50.69 ± 0.49~~\textcolor{mygreen}{$\downarrow$0.81} & 
\textbf{85.02 ± 0.26}~~\textcolor{red}{$\uparrow$2.60} & 
\textbf{29.97 ± 1.02}~~\textcolor{red}{$\uparrow$2.27} & 
\textbf{35.16 ± 1.18}~~\textcolor{red}{$\uparrow$2.99} & 
\textbf{3.5}~~\textcolor{mygreen}{$\downarrow\times$1.9} \\
\hline
\end{tabular}}
\caption{Transfer learning accuracy comparison on various datasets with average per-sample evaluation latency. The gaps indicated by \textcolor{red}{$\uparrow$} and \textcolor{mygreen}{$\downarrow$} represent the difference relative to the second-best method (excluding \ourmethodo) on the current task.}
\label{tab:zero_shot_performance}
\end{table*}

\begin{table*}[t]
\centering
\resizebox{\linewidth}{!}{
\begin{tabular}{llllllll}
    \hline
    & \multicolumn{2}{c}{\textbf{Omniglot}} & \multicolumn{2}{c}{\textbf{Mini-ImageNet}} & \multicolumn{2}{c}{\textbf{Tiered-ImageNet}} & \\
    \cmidrule(lr){2-3}  
    \cmidrule(lr){4-5}
    \cmidrule(lr){6-7}
    \textbf{Method} 
    & \textbf{(5, 1)} & \textbf{(5, 5)} 
    & \textbf{(5, 1)} & \textbf{(5, 5)} 
    & \textbf{(5, 1)} & \textbf{(5, 5)} 
    & \textbf{Latency (ms)} \\
    \hline 
    REPTILE~\cite{REPTILE} & 95.39 ± 0.09 & 98.90 ± 0.10 & 47.07 ± 0.26 & 62.74 ± 0.37 & 49.12 ± 0.43 & 65.99 ± 0.42 & 20.7 \\
    Meta-Baseline~\cite{meta_baseline} & \textbf{97.75 ± 0.25} & \textbf{99.68 ± 0.18} & 58.10 ± 0.31 & 74.50 ± 0.29 & 68.62 ± 0.29 & 83.29 ± 0.51 & 19.4 \\
    Meta-Hypernetwork~\cite{Meta-Hypernetwork} & 96.57 ± 0.22 & 98.83 ± 0.16 & 52.50 ± 0.28 & 67.76 ± 0.34 & 53.80 ± 0.35 & 69.98 ± 0.42 & 13.1 \\
    GHN3~\cite{GHN3}       & 95.23 ± 0.23 & 98.65 ± 0.19 & 63.22 ± 0.29 & 76.79 ± 0.33 & 64.72 ± 0.36 & 78.40 ± 0.46 & 17.5 \\
    OCD~\cite{OCD}         & 95.04 ± 0.18 & 98.74 ± 0.14 & 59.76 ± 0.27 & 75.16 ± 0.35 & 60.01 ± 0.38 & 76.33 ± 0.47 & 8.4 \\
    Meta-Diff~\cite{MetaDiff}    & 94.65 ± 0.65 & 97.91 ± 0.53 & 55.06 ± 0.81 & 73.18 ± 0.64 & 57.77 ± 0.90 & 75.46 ± 0.69 & 8.9 \\
    D2NWG~\cite{D2NWG}    & 96.77 ± 6.13 & 98.94 ± 7.49 & 61.13 ± 8.50 & 76.94 ± 6.04 & 65.33 ± 6.50 & 85.05 ± 8.25 & 10.4 \\
    \ourmethodo (ours) & 96.65 ± 0.19 & 99.34 ± 0.23 & 62.53 ± 0.37 & 76.25 ± 0.28 & 64.72 ± 0.17 & 83.26 ± 0.49 & 6.4 \\
    \rowcolor{myblue} \ourmethod (ours) & 
    97.34 ± 0.19~~\textcolor{mygreen}{$\downarrow$0.41} & 
    99.14 ± 0.13~~\textcolor{mygreen}{$\downarrow$0.54} & 
    \textbf{64.87 ± 0.32}~~\textcolor{red}{$\uparrow$1.65} & 
    \textbf{79.98 ± 0.18}~~\textcolor{red}{$\uparrow$3.04} & 
    \textbf{70.45 ± 0.27}~~\textcolor{red}{$\uparrow$1.83} & 
    \textbf{88.58 ± 0.21}~~\textcolor{red}{$\uparrow$3.53} & 
    \textbf{4.8}~~\textcolor{mygreen}{$\downarrow\times$1.8} \\
    \hline
\end{tabular}}
\caption{Few-shot task accuracy comparison on Omniglot, Mini-ImageNet, and Tiered-ImageNet datasets with average per-sample evaluation latency.}
\label{tab:few_shot_performance_latency}
\end{table*}

\subsubsection{Few-Shot Learning}\label{sec:few_shot_learning} 
\noindent\textbf{Task.} Following the setup provided by~\citep{MAML}, we train and evaluate models on disjoint meta-training and meta-testing tasks. During the evaluation stage, we use the support set in meta-test tasks to fine-tune the models, and then measure the accuracy and average per-sample evaluation latency on the query set.
\par
\noindent\textbf{Dataset.} We use Omniglot~\cite{omniglot}, Mini-ImageNet~\cite{miniImagenet}, and Tiered-ImageNet~\cite{Tiered-ImageNet} datasets for the construction of 5-way 1-shot, 5-way 5-shot tasks. To evaluate generalization capabilities, we maintain distinct and separate class sets for training and evaluation phases.
\par
\noindent\textbf{Baselines.}\label{sec:zero-shot}
Our benchmarks consist of gradient-based methods, including REPTILE~\cite{REPTILE} and Meta-Baseline~\cite{meta_baseline}, and weight generation-based methods, including Meta-Hypernetwork~\cite{Meta-Hypernetwork}, GHN3~\cite{GHN3}, OCD~\cite{OCD}, Meta-Diff~\cite{MetaDiff}, and D2NWG~\cite{D2NWG}. Following the setting given by MAML~\cite{MAML}, the downstream neural network uses four convolution blocks with a linear probe for classification.
\par
\noindent\textbf{Results.}
Table~\ref{tab:zero_shot_performance} shows that \ourmethod can improve performance on almost all tasks. Since the 5-way task of Omniglot is relatively easy to learn, the gradient-based method Meta-Baseline can achieve slightly higher accuracy. On the winning tasks, compared to the second-best method, \ourmethod achieves an average improvement of 2.51\% in accuracy. Compared to the current fastest weight generation algorithm, \ie, OCD, \ourmethod achieves a 1.8$\times$ reduction in inference latency. Note that the gradient-based methods, REPTILE and Meta-Baseline, exhibit high latency here due to the need for gradient computation during fine-tuning.

\subsubsection{Domain Generalization}\label{sec:multi-domain}
\noindent\textbf{Task.} In this task, we explore the domain generalizability of \ourmethod. We follow the domain generalization few-shot task given by~\citep{hierar_meta} to evaluate the model's performance.
\par
\noindent\textbf{Dataset.} We use DomainNet~\cite{DomainNet} for the construction of 5-way 1-shot and 20-way 5-shot tasks. Specifically, we use Clipart, Infograph, Painting, Quickdraw, and Real domains for meta-training, while Sketch domains for meta-testing. Under this setting, the tasks in the meta-training set may come from different domains, and the tasks in the meta-testing set may come from another unseen domain.
\par
\noindent\textbf{Baselines.}
We benchmark against REPTILE, Meta-baseline, Meta-Hypernetwork, GHN3, OCD, Meta-Diff, and D2NWG. The downstream network uses ResNet12 with a linear probe for classification.
\par
\noindent\textbf{Results.}
Table~\ref{tab:multi_domain_performance} shows that \ourmethod significantly outperforms current methods on few-shot domain generalization tasks. Compared to the second-best baselines in each task, \ourmethod achieved an average improvement of 4.64\% in accuracy. It can be observed that \ourmethod achieves a significant performance improvement over other methods, owing to its integration of meta-learning with weight generation and its maintenance of local–global consistency. In terms of overhead, \ourmethod achieves a 1.7$\times$ reduction in latency compared to OCD, showing the same advantage as in transfer learning and few-shot learning tasks.

\begin{table}[t]
\centering
\begin{minipage}{0.48\linewidth}
\centering
\vspace{-5pt}
\resizebox{\linewidth}{!}{
\begin{tabular}{llll}
    \toprule
    &  \multicolumn{2}{c}{\textbf{DomainNet}} &\\
    \cmidrule{2-3}
    \textbf{Method} & \textbf{(5, 1)} & \textbf{(20, 5)} & \textbf{Latency (ms)} \\
    \midrule
    REPTILE~\cite{REPTILE}          & 48.13 ± 0.50 & 52.32 ± 0.42 & 26.8 \\
    Meta-Baseline~\cite{meta_baseline}         & 50.54 ± 0.47 & 54.45 ± 0.40 & 26.2 \\
    Meta-Hypernetwork~\cite{Meta-Hypernetwork}    & 59.00 ± 0.39 & 63.32 ± 0.35 & 10.3 \\
    OCD~\cite{OCD} & 64.58 ± 0.42 & 67.10 ± 0.38 & 9.7  \\
    GHN3~\cite{GHN3}             & 63.11 ± 0.36 & 66.42 ± 0.33 & 30.1 \\
    Meta-Diff~\cite{MetaDiff}              & 64.24 ± 0.41 & 67.58 ± 0.39 & 10.9 \\
    \ourmethodo (ours)        & 67.04 ± 0.29 & 70.15 ± 0.41 & 7.5 \\
    \rowcolor{myblue}
    \ourmethod (ours) &
    \textbf{69.05 ± 0.39}~~\textcolor{red}{$\uparrow$4.47} &
    \textbf{72.86 ± 0.31}~~\textcolor{red}{$\uparrow$5.28} &
    \textbf{5.8}~~\textcolor{mygreen}{$\downarrow\times1.7$} \\
    \bottomrule
\end{tabular}}
\caption{Accuracy and average per-sample evaluation latency on DomainNet 5-way 1-shot and 20-way 5-shot task.}
\label{tab:multi_domain_performance}
\end{minipage}
\hfill
\begin{minipage}{0.51\linewidth}
\centering
\resizebox{\linewidth}{!}{
\begin{tabular}{lllll}
    \toprule
    & \multicolumn{2}{c}{\textbf{MRPC}}
    & \multicolumn{2}{c}{\textbf{QNLI}}\\
    \cmidrule(lr){2-3} \cmidrule(lr){4-5}
    \textbf{Method} & \textbf{Acc} & \textbf{Latency (h)} & \textbf{Acc} & \textbf{Latency (h)} \\
    \midrule
    Full-fine-tune   & \textbf{90.24 ± 0.57} & 1.47 & 92.84 ± 0.26 & 3.15 \\
    LoRA~\cite{LoRa}            & 89.76 ± 0.69 & 0.81 & \textbf{93.32 ± 0.20} & 1.76 \\
    AdaLoRA~\cite{AdaLoRA}         & 88.71 ± 0.73 & 0.74 & 93.17 ± 0.25 & 1.68 \\
    DyLoRA~\cite{DyLoRA}          & 89.59 ± 0.81 & 0.76 & 92.21 ± 0.32 & 1.63 \\
    FourierFT~\cite{FourierFT}       & 90.03 ± 0.54 & 0.68 & 92.25 ± 0.15 & 1.55 \\
    \ourmethodo (ours) & 88.73 ± 0.64 & 0.22 & 91.04 ± 0.23 & 0.48 \\
    \rowcolor{myblue}
    \ourmethod (ours) 
      & 89.43 ± 0.70~~\textcolor{mygreen}{$\downarrow$0.60} 
      & \textbf{0.19}~~\textcolor{mygreen}{$\downarrow\times$3.6} 
      & 91.86 ± 0.29~~\textcolor{mygreen}{$\downarrow$1.31} 
      & \textbf{0.41}~~\textcolor{mygreen}{$\downarrow\times$4.0} \\
    \bottomrule
\end{tabular}}
\caption{Accuracy and fine-tuning latency comparison on MRPC and QNLI tasks for different fine-tuning algorithms.}
\label{tab:performance_fine-tune_time_comparison}
\end{minipage}
\end{table}

\subsubsection{Large Language Model fine-tuning}
\noindent\textbf{Task.}
We demonstrate that \ourmethod can be applied to the fine-tuning of Large Language Models by learning to generate LoRA~\cite{LoRa} matrices for new tasks. We compared the algorithms in terms of their fine-tuning accuracy upon convergence and the latency required to achieve it.
\par
\noindent\textbf{Datasets.}
We conduct a case study to demonstrate the generalizability and efficiency of \ourmethod. We use five binary classification tasks, \ie, SST-2, QQP, RTE, WNIL, and CoLA from the GLUE~\cite{GLUE} benchmark for pre-training. Then we use the other two tasks, \ie, MRPC and QNIL, to evaluate the performance of the methods.
\par
\noindent\textbf{Baselines.}
We benchmark against Full-fine-tuning baseline, LoRA~\cite{LoRa}, AdaLoRA~\cite{AdaLoRA}, DyLoRA~\cite{DyLoRA}, and FourierFT~\cite{FourierFT}, which are all gradient-based fine-tuning algorithms. The large language model we fine-tuned is RoBERTa-base~\cite{FourierFT}. The LoRA matrices are generated following the fine-tuning process given by~\citep{FourierFT}. Note that \ourmethod is a meta-learning-based method that can learn on all training tasks, while other baselines are single-task fine-tuning methods that only work on one task. Note that \ourmethod requires extra time to pre-train the diffusion model $f^G_{\phi}$; however, this is a one-time cost and demonstrates better potential in multi-task fine-tuning scenarios.
\par
\noindent\textbf{Results.}
Table~\ref{tab:performance_fine-tune_time_comparison} shows that \ourmethod achieves comparable binary classification accuracy on two evaluation tasks compared to other gradient-based fine-tuning algorithms while significantly accelerating the fine-tuning speed by $\times$3.6 to $\times$4.0. By incorporating meta-learning, \ourmethod efficiently captures shared representations from pre-training tasks, enabling the direct, gradient-free generation of task-specific LoRA matrices.

\section{Conclusion and Limitation}\label{sec:conclusion}
\noindent\textbf{Conclusion.} In this paper, we propose \ourmethod, which integrates the fast inference ability of weight generation technology and the cross-task transferability of bi-level optimization. Building on this, we further propose local consistency diffusion, enabling the model to capture optimization policy in local target weights. It employs a step-by-step approach to assign local targets, elegantly preserving consistency with the global optima. In addition, we theoretically and empirically demonstrate that the convergence of the weight generation paradigm can be improved by introducing SAM. We validate \ourmethod's superior accuracy and inference efficiency in tasks that require frequent weight updates, such as transfer learning, few-shot learning, domain generalization, and large language model adaptation.

\noindent\textbf{Limitation and future work.} Our method builds on the existing diffusion-based weight generation framework. Diffusion models typically depend on a fixed inference starting point—namely, Gaussian noise—which limits the model’s ability to capture local supervisory signals. As shown in Figure~\ref{fig:consistency}, a more intuitive way to improve weight generation is to learn arbitrary sub-trajectories $\theta_i \to \theta_j$ rather than fix trajectories $\theta_0 \to \theta_j$. In future work, we will adopt more general techniques, such as Flow Matching~\cite{flow_match} and Trajectory Balance~\cite{GFLOW1,GFLOW2,GFLOW3,GFLOW4,GFLOW5}, to model the weight generation problem and break through this limitation.

\clearpage
\bibliographystyle{plain} 
\bibliography{reference}

\clearpage
\setcounter{page}{1}
\setcounter{theorem}{0}
\setcounter{assumption}{0}
\setcounter{equation}{0}
\setcounter{section}{0}

\clearpage
\appendix
\section{Theorem 1 and Proof}\label{sec:appendix_1}
Readers can refer to the derivation process of DDPM~\cite{understand_DDPM} to understand the following derivation. To align with the subscript ordering of weights $\theta_0 \to \theta_M$, we reverse the index order of diffusion states $x_i$ mentioned in~\cite{DDPM,understand_DDPM}.

\begin{theorem}\label{the:L_k}
Given the number of diffusion steps $T$, an increasing noise schedule $\{\alpha_0,...,\alpha_T\}$, local target weights $\{\theta_d,...,\theta_{k* d}\}$, and let the inference process align with the vanilla diffusion algorithm, \ie, 
\begin{equation}\label{eq:DDPM_inference}
x_{t+1} = \frac{1}{\sqrt{\bar{\alpha}_{t+1}}} \left( x_t - \sqrt{1 - \bar{\alpha}_{t+1}} \epsilon_\theta(x_t, t) \right).
\end{equation}

Then, the denoiser $\epsilon_{\phi}$ can recover the target sequence $\{\theta_d,...,\theta_{k* d=M}\}$ from standard Gaussian noise $x_0$ with evenly $T/k$ steps intervals, when $\epsilon_{\phi}$ is trained by local consistency loss $L^{loc}$ as follows:
\begin{align}\label{eq:consistency_loss}
&L^{loc}=\underset{i\in(0,k]}{E}~L^{loc}_i,\notag\\
&L_i^{loc}=\underset{t\in[0,i* T/k)}{E}||\sqrt{1-\bar{\alpha}_t^i}\epsilon_{\phi}(x_t, t)-\sqrt{1-\bar{\alpha}_t}\epsilon||^2,\\
&x_t=\sqrt{\bar{\alpha}^i_t} \theta_{i* d} + \sqrt{1 - \bar{\alpha}^i_t}\epsilon, \notag
\end{align}
where $\bar{\alpha}_t=\prod_{j=t}^{T-1} \alpha_j$,  $\bar{\alpha}^i_t=\prod_{j=t}^{i* T/k-1} \alpha_j$, and $\epsilon$ denotes standard Gaussian noise.
\end{theorem}

\begin{proof}
First, consider a single local target $\theta_{i*d=i*M/k}=x_{i*T/k}$, its inference chain $\{x_0,...,x_{i* T/k}\}$ and its increasing schedule  $\{\alpha_{0},...,\alpha_{i*T/k}\}$. Let $x_{i*T/k}$ be $x_{T'}$ and diffusion steps be $i*T/K$.

According to vanilla diffusion model, to maximize the likelihood $p(x_T')$ of observed data $x_T'$, we need to minimize denoising matching term
\begin{equation}\label{eq:KL divergence}
     \underset{t\in[0,T'), q(x_t | x_{T'})}{{E}} \left[ \mathrm{D}_{\mathrm{KL}}\left( q(x_{t+1} | x_{t}, x_T') \parallel p_{\phi}(x_{t+1} | x_{t}) \right) \right].
\end{equation}
In the KL divergence bracket, the left term can be expanded by the Bayesian Theorem. The right term is the inference process to be modeled with $\phi$, whose expectation is given by Equation~\ref{eq:DDPM_inference}. 

According to Bayes Theorem,
\begin{align}\label{eq:bayesian}
    q(x_{t+1} \mid x_{t}, x_{T'}) = \frac{q(x_{t} \mid x_{t+1}, x_{T'}) \, q(x_{t+1} \mid x_{T'})}{q(x_{t+1} \mid x_{T'})}.
\end{align}

According to the Markov Rule and standard diffusion process, 
\begin{align}\label{eq:4}
    q(x_{t} \mid x_{t+1}, x_{T'}) &= q(x_t \mid x_{t+1}) \notag\\
    &\sim \mathcal{N}(x_t;\sqrt{\alpha_t}x_{t+1}+\sqrt{1-\alpha_t}\mathbf{I}).
\end{align}

Recursively using the diffusion process on $\{x_{T'},..., x_{t+1}\}$, we have
\begin{equation}\label{eq:5}
    q(x_t|x_{T'}) \sim \mathcal{N}(x_t; \sqrt{{\bar{\alpha}_t^{T'}}} \, x_{T'}, (1 - {\bar{\alpha}_t^{T'}}) \, \mathbf{I}),
\end{equation}
where ${\bar{\alpha}_t^{T'}}=\prod^{{T'-1}}_{j=t} \alpha_j$. Note that the coefficients here differ from those in the vanilla diffusion algorithm. 

According to Equation~\ref{eq:5}, $q(x_{t+1}|x_{T'})$ can be written as
\begin{equation}\label{eq:6}
    q(x_{t+1}|x_{T'}) \sim \mathcal{N}(x_{t+1}; \sqrt{{\bar{\alpha}_{t+1}^{T'}}} \, x_{T'}, (1 - {\bar{\alpha}_{t+1}^{T'}}) \, \mathbf{I}).
\end{equation}

According to Equation~\ref{eq:4}~\ref{eq:5}~\ref{eq:6}, Equation~\ref{eq:bayesian} can be written as

\begin{align*}
  & q(x_{t+1} \mid x_t, x_{T'}) 
    = \frac{q(x_t \mid x_{t+1}, x_{T'}) \, q(x_{t+1} \mid x_{T'})}
           {q(x_t \mid x_{T'})}\\
  \sim\;& 
    \frac{\mathcal{N}\bigl(x_t; \sqrt{\alpha_t}\,x_{t+1}, (1-\alpha_t)\mathbf{I}\bigr)\,
          \mathcal{N}\bigl(x_{t+1}; \sqrt{\bar{\alpha}_{t+1}^{T'}}\,x_{T'}, (1-\bar{\alpha}_{t+1}^{T'})\mathbf{I}\bigr)}
         {\mathcal{N}\bigl(x_t; \sqrt{\bar{\alpha}_t^{T'}}\,x_{T'}, (1-\bar{\alpha}_t^{T'})\mathbf{I}\bigr)}\\
\propto\;& 
    \exp\!\Biggl\{-\tfrac12 
      \Bigl(\tfrac{1}{\tfrac{(1-\alpha_t)\,(1-\bar{\alpha}_{t+1}^{T'})}
                             {1-\bar{\alpha}_t^{T'}}}\Bigr)\,
      \Bigl[
        x_{t+1}^2 
        - 2\,\frac{\sqrt{\alpha_t}\,(1-\bar{\alpha}_{t+1}^{T'})\,x_t
                 + \sqrt{\bar{\alpha}_{t+1}^{T'}}\,(1-\alpha_t)\,x_{T'}}
                 {1-\bar{\alpha}_t^{T'}}\,x_{t+1}
      \Bigr]
    \Biggr\}.
\end{align*}

According to the definition of Gaussian distribution, the variance of Equation~\ref{eq:bayesian} can be written as
\[
\sigma_q(t)\propto\frac{(1-\alpha_t)(1-\bar\alpha_{t+1}^{T'})}{1-\bar\alpha_t^{T'}}\mathbf{I},
\]
the expectation can be written as
\begin{equation}\label{eq:expectation}
\mu(x_t,x_{T'})\propto\frac{\sqrt{\alpha_t}(1-\bar\alpha_{t+1}^{T'})x_t+\sqrt{\bar\alpha_{t+1}^{T'}}(1-\alpha_t)x_{T'}}{1-\bar\alpha_t^{T'}}.
\end{equation}

Recursively apply the reparameterization trick to Equation~\ref{eq:4} in an iterative manner, we have:\footnote{This step is relatively complex; it is recommended to refer to~\citep{understand_DDPM}.}
\begin{equation}\label{eq:reparam}
\textcolor{red}{
x_t=\sqrt{\bar\alpha_t^{T'}}x_{T'}+\sqrt{1-\bar\alpha_t^{T'}}\epsilon.}
\end{equation}

Reorganize Equation~\ref{eq:reparam}, we have 
\[
x_{T'}=\frac{1}{\sqrt{\bar\alpha_t^{T'}}}\bigl(x_t-\sqrt{1-\bar\alpha_t^{T'}}\,\epsilon\bigr).
\]

Substitute $x_{T’}$ into Equation~\ref{eq:expectation}, we have

\begin{align}\label{eq:mu estimate}
\mu(x_t,x_{T'})&\propto\frac{\sqrt{\alpha_t}(1-\bar\alpha_{t+1}^{T'})x_t+\sqrt{\bar\alpha_{t+1}^{T'}}(1-\alpha_t)\frac{1}{\sqrt{\bar\alpha_t^{T'}}}(x_t-\sqrt{1-\bar\alpha_t^{T'}}\epsilon)}{1-\bar\alpha_t^{T'}}\nonumber\\
&=\frac{1}{\sqrt{\alpha_t}}x_t-\frac{1-\alpha_t}{\sqrt{1-\bar\alpha_t^{T'}}\sqrt{\alpha_t}}\epsilon.
\end{align}

Returning to Equation~\ref{eq:KL divergence}, minimizing KL divergence is equivalent to minimizing the difference between the expectations of the two terms in the bracket. Equation~\ref{eq:DDPM_inference} gives the expectation of the right term, \ie,
\[
x_{t+1} = \frac{1}{\sqrt{\alpha_t}}x_t - \frac{1 - \alpha_t}{\sqrt{1 - \bar{\alpha}_t}\,\sqrt{\alpha_t}}\,\epsilon_{\phi},
\]
where \textcolor{red}{$\bar{\alpha}_t=\prod_{j=t}^{T-1} \alpha_j$}.

Equation~\ref{eq:mu estimate} is the expectation of the left term. As a result, Equation~\ref{eq:KL divergence} can be simplified to the form of the left-term' expectation minus the right-term's expectation:

\begin{align}\label{eq:L_T_loc}
L_{T'}&=\underset{t\in[0,T')}{{E}}||\mu(x_t,x_{T'})-\mu_{\phi}(x_t,t)||^2\notag\\
&=\underset{t\in[0,T')}{{E}}\frac{(1-\alpha_t)^2}{\alpha_t(1-\bar{\alpha}_t^{T'})(1-\bar{\alpha}_t)}||\sqrt{1-\bar{\alpha}_t^{T'}}\epsilon_{\phi}-\sqrt{1-\bar{\alpha}_t}\epsilon_{T'}||^2\notag\\
&\propto\underset{t\in[0,T')}{{E}}||\sqrt{1-\bar{\alpha}_t^{T'}}\epsilon_{\phi}-\sqrt{1-\bar{\alpha}_t}\epsilon||^2.
\end{align}

Note that \textcolor{red}{${\bar{\alpha}_t^{T'}}=\prod^{{T'-1}}_{j=t}\alpha_j=\prod^{{i*T/k-1}}_{j=t}\alpha_j=\bar{\alpha}^i_t$}. So we substitute $T'=i*T/k$ into to Equation~\ref{eq:L_T_loc}, it yield:
\textcolor{red}{
$$
    \underset{t\in[0,i* T/k)}{E}||\sqrt{1-\bar{\alpha}_t^i}\epsilon_{\phi}(x_t, t)-\sqrt{1-\bar{\alpha}_t}\epsilon||^2
$$}
which is consistent with the intended formulation $L^{loc}_{i}$ in Equation~\ref{eq:consistency_loss}. 

This means that $L^{loc}_i$ enables the diffusion model to generate the inference chain from Gaussian noise to $\theta_{i*d=i*M/k}$ in $i*T/k$ steps. 

Moreover, since the inference process is fixed, \ie, Equation.~\ref{eq:DDPM_inference}, $L^{loc}=\underset{i\in(0,k]}{E}L^{loc}_i$ allows any $\theta_{i*d}$ and $\theta_{(i+1)*d}$ to share the same inference chain. Their interval is \textcolor{red}{$$(i+1)*T/k-i*T/k=T/k.$$} 
Thus, the proof is complete.
\end{proof}

\clearpage
\section{Proposition 1 and Proof}\label{sec:appendix_propo}
\begin{proposition}
When $k=1$, local consistency diffusion is equivalent to the vanilla diffusion algorithm.
\end{proposition}
\begin{proof}
When $k=1$, $d = M/k = M$, so $i\in(0,k] \to i=1$. At this point, only the optimal weight $\theta_{i\cdot d=M}$ is modeled. 

Moreover, for $k=1$, we have
\[
\bar{\alpha}_t^i = \prod_{j=t}^{i*T/k-1}\alpha_j
= \prod_{j=t}^{T-1}\alpha_j
= \bar{\alpha}_t,
\]
hence the two coefficients in Equation~\ref{eq:consistency_loss} are identical and can be simplified, making $L_i$ equivalent to $L^{loc}_i$. This makes local consistency diffusion be equivalent to the vanilla diffusion algorithm.
\end{proof}

\clearpage
\section{Theorem 2 and Proof}\label{sec:appendix_2}

\begin{lemma}
\label{lemma:theta_convergence}
Assume that the loss function \( L^D(\cdot) \) is \( l \)-smooth and satisfies \( \mu \)-strongly convex. Then, the sequence \( \{\theta_{i}\}_{i=0}^{M} \) generated by the gradient descent update with step size \( \frac{1}{l} \) satisfies
\[
\|\theta_{M} - \theta_*\|^2 \leq \frac{2 [L^D(\theta_0) - L^D(\theta_*)]}{\mu} \left(1 - \frac{\mu}{l}\right)^{M}.
\]
\end{lemma}

\begin{proof}
Since \( L^D(\theta) \) is \( l \)-smooth, for any \( \theta \) and \( \theta' \),
\begin{equation*}
    L^D(\theta') \leq L^D(\theta) + \nabla L^D(\theta)^\top (\theta' - \theta) + \frac{l}{2} \|\theta' - \theta\|^2.
\end{equation*}
Applying this to the gradient descent update $\theta_{k+1} = \theta_k - \frac{1}{l} \nabla L^D(\theta_k)$, we have

\begin{align}\label{eq:expend}
L^D(\theta_{k+1}) &\leq L^D(\theta_k) + \nabla L^D(\theta_k)^\top (\theta_{k+1} - \theta_k) + \frac{l}{2} \|\theta_{k+1} - \theta_k\|^2 \notag\\
&= L^D(\theta_k) - \frac{1}{2l} \|\nabla L^D(\theta_k)\|^2.
\end{align}

Since  $L^D(\theta)$ is $\mu$-strongly convex, it satisfies the Polyak--Lojasiewicz condition:
\begin{equation*}
\frac{1}{2} \|\nabla L^D(\theta)\|^2 \geq \mu [L^D(\theta) - L^D(\theta_*)].
\end{equation*}
Substituting this inequality into Equation~\ref{eq:expend}, we have
\begin{align*}
L^D(\theta_{k+1}) & \leq L^D(\theta_k) -  \frac{\mu}{l}  (L^D(\theta_k) - L^D(\theta_*)).
\end{align*}
The above equation can be reorganized to
\[
L^D(\theta_{k+1}) - L^D(\theta_*) \leq \left( 1 - \frac{\mu}{l} \right) (L^D(\theta_k) - L^D(\theta_*)).
\]
Start from $k=0$, and recursively apply the above equation with $M$ times. It follows that
\[
L^D(\theta_{M}) - L^D(\theta_*) \leq \left( 1 - \frac{\mu}{l} \right)^{M} (L^D(\theta_0) - L^D(\theta_*)).
\]
Since \( L^D(\theta) \) is \( \mu \)-strongly convex, it satisfies
\[
\|\theta - \theta_*\|^2 \leq \frac{2}{\mu} (L^D(\theta) - L^D(\theta_*)).
\]
So we have
\begin{align*}
\|\theta_{M} - \theta_*\|^2 &\leq \frac{2}{\mu} (L^D(\theta_{M}) - L^D(\theta_*)) \\
&\leq \frac{2(L^D(\theta_0) - L^D(\theta_*))}{\mu}  \left( 1 - \frac{\mu}{l} \right)^{M}.
\end{align*}
\end{proof}

\begin{theorem}
Assume that the reconstruction error of the generative model is bounded by \( c \), the downstream loss is bounded by $\psi$, and the loss function is both \( l \)-smooth and \( \mu \)-strongly convex, with the eigenvalues of the Hessian matrix around the optimum \( \theta_* \) bounded by \( \lambda \). Then, the cumulative empirical error of the weight generation paradigm can be bounded as follows:
\begin{equation}
    L^D(\hat{\theta})-L^D(\theta_*) \leq \frac{\lambda}{2} \left[c+\frac{2\psi}{\mu}\left(1 - \frac{\mu}{l}\right)^{M}\right],
\end{equation}
where $\hat{\theta}$ is the weight predicted by the generative model.
\end{theorem}

\begin{proof}
Using the Taylor expansion around the optimal point \( \theta_* \), we have

\begin{align}\label{eq:taylor}
L^D(\hat{\theta}) - L^D(\theta_*) &= \nabla L^D(\theta_*)^T (\hat{\theta} - \theta_*) + \frac{1}{2} (\hat{\theta} - \theta_*)^T \nabla^2 L^D(\xi) (\hat{\theta} - \theta_*)\notag\\
&=\frac{1}{2} (\hat{\theta} - \theta_*)^T \nabla^2 L^D(\xi) (\hat{\theta} - \theta_*).
\end{align}

According to a constraint on the Hessian matrix, we have
\begin{equation}\label{eq:hessiam max value}
 \frac{1}{2} (\hat{\theta} - \theta_*)^T \nabla^2 L^D(\xi) (\hat{\theta} - \theta_*) \leq \frac{\lambda}{2}  ||\hat{\theta} - \theta_*||^2.
\end{equation}

Decomposing $|\hat{\theta} - \theta_*||^2$ into weight preparation error and reconstruction error, we have
\begin{align}\label{eq:theta error cumulative}
    ||\hat{\theta}-\theta_*||^2 &\leq ||\hat{\theta}-\theta_{M}||^2 + ||\theta_{M}-\theta_*||^2 \notag\\
&\leq c + \frac{2 \left(L^D(\theta_0) - L^D(\theta_*)\right)}{\mu} \left(1 - \frac{\mu}{l}\right)^{M} \notag(\text{Using Lemma~\ref{lemma:theta_convergence}})\\
& \leq c + \frac{2 \psi}{\mu} \left(1 - \frac{\mu}{l}\right)^{M}.
\end{align}

Substituting Equation~\ref{eq:theta error cumulative} and Equation~\ref{eq:hessiam max value} into Equation~\ref{eq:taylor} we obtain
\textcolor{red}{
$$
    L^D(\hat{\theta})-L^D(\theta_*) \leq \frac{\lambda}{2} \left[c+\frac{2\psi}{\mu}\left(1 - \frac{\mu}{l}\right)^{M}\right].
$$}
\end{proof}

\clearpage
\section{Preliminary}\label{sec:appendix_3}
\subsection{Diffusion Model}
The core idea of the diffusion model is to model data through a two-stage process:
\begin{itemize}
    \item \textbf{Diffusion Process}: Starting with data $x_T$, noise is added at each step to generate $x_0$, eventually approaching a standard normal distribution.
    \item \textbf{Inference Process}: Starting with noise $x_0$, a denoising model $\phi$ generates $x_{1}, x_{2}, \dots, x_T$ step by step.
\end{itemize}

By precisely modeling the reverse process, the diffusion model $\phi$ can generate new samples that match the original data distribution. Diffusion models define an increasing schedule $\{\alpha_i\}_{i=0}^T$ to control the noise level at each step. In the diffusion process, noise is added at each step $t$, transforming the data $x_{t+1}$ into $x_t$
\[
q(x_t | x_{t+1}) = \mathcal{N}(x_t; \sqrt{\alpha_t} x_{t+1}, (1-\alpha_t) \mathbf{I}).
\]
The direct transition from $x_T$ to $x_t$ can be written as
\[
q(x_t | x_T) = \mathcal{N}(x_t; \sqrt{\bar{\alpha}_t} x_T, (1 - \bar{\alpha}_t) \mathbf{I}),
\]
where $\bar{\alpha}_t = \prod_{i=t}^{T-1} \alpha_i$ is the cumulative noise schedule, controlling the overall noise level from $x_0$ to $x_t$.
In the inference process, the denoiser iteratively reconstructs the data by 
$$
x_{t+1}=\frac{1}{\sqrt{\alpha_t}}x_t - \frac{1 - \alpha_t}{\sqrt{1 - \bar{\alpha}_t} \sqrt{\alpha_t}}\epsilon_{\phi}+\sigma_z\epsilon,
$$
We omit the additional $\sigma_z\epsilon$ for stable weight generation. The key point of training a variational model is to maximize the Evidence Lower Bound (ELBO). In the diffusion algorithm, optimizing the ELBO is essentially equivalent to minimizing the denoising match term
$$
     \underset{t\in[0,T'), q(x_t | x_{T'})}{{E}} \left[ \mathrm{D}_{\mathrm{KL}}\left( q(x_{t+1} | x_{t}, x_T') \parallel p_{\phi}(x_{t+1} | x_{t}) \right) \right].
$$
which is also the objective optimized by our local consistency diffusion.
\subsection{REPTILE}
REPTILE is a first-order optimization-based meta-learning algorithm that simplifies training while retaining strong adaptability across tasks. It eliminates the need for second-order gradients, making it computationally efficient compared to algorithms like MAML. The training process consists of two loops: the inner-loop and the outer-loop.

In the inner loop, REPTILE performs gradient descent on a sampled task \( {T}_i \) using the task’s support set. Starting from the meta-parameters \( \phi \), the task-specific parameters \( \phi_i \) are updated for \( K \) steps using
\[
\phi_i^{(t+1)} = \phi_i^{(t)} - \eta \nabla_{\phi_i^{(t)}} {L}_{{T}_i}(\phi_i^{(t)}),
\]
where \( \eta \) is the inner-loop learning rate and \( {L}_{{T}_i} \) is the loss for the task\( {T}_i \). 

In the outer-loop, the meta-parameters \( \phi \) are updated by moving them toward the task-specific parameters \( \phi_i \) obtained from the inner loop. This meta-update is given by
\[
\phi \leftarrow \phi+ \zeta (\phi_i - \phi),
\]
where \( \zeta \) is the outer-loop learning rate. 

By iteratively repeating the inner and outer loops across multiple tasks drawn from the task distribution $\mathcal{T}$ , REPTILE optimizes the meta-parameters \( \phi \) to find an initialization that enables fast adaptation to new tasks with minimal gradient steps. Its simplicity lies in avoiding second-order derivatives, while its effectiveness is demonstrated across diverse applications such as few-shot learning and domain generalization.

\clearpage
\section{Setup Details}\label{sec:appendix_4}
\subsection{Dataset}\label{sec:dataset setup}
\noindent\textbf{Omniglot.}
The raw Omniglot dataset contains 1623 handwritten characters from 50 alphabets, each with 20 instances in 28$\times$28 grayscale format. We partition the classes of the training set, evaluation set, and testing set into 800:400:432. We use Omniglot in three scenarios. We used the Omniglot dataset in our preliminary experiments, ablation experiments, and comparative experiments. For the construction of the classification task, we referred to the experimental setup by MAML.
\\
\par
\noindent\textbf{Mini-ImageNet.}
The raw Mini-ImageNet contains 100 classes, each containing 600 instances in 84$\times$84 grayscale format. We partition classes of training set, evaluation set, and testing set into 64:16:20. The usage of Mini-ImageNet is the same as Omniglot, and we also follow the setup given by MAML.
\\
\par
\noindent\textbf{Tiered-ImageNet.}
The Tiered-Imagenet dataset is a larger-scale few-shot learning benchmark derived from ImageNet, containing 608 classes grouped into 34 higher-level categories. Each image is in 84$\times$84 resolution with RGB channels. Following the standard protocol (Meta-Baseline), we partition the dataset into 351 classes for training, 97 for validation, and 160 for testing, ensuring no class overlap across phases. We use Tiered-ImageNet in all experimental stages, including ablation and comparative evaluations, to assess the generalization ability of our method on a more diverse and semantically structured dataset.
\\
\par
\noindent\textbf{ImageNet-1K.}
The raw ImageNet-1K is a benchmark dataset with 1000 classes, 1.2 million training images, and 50000 validation images, typically resized to a resolution of 224$\times$224 pixels. We partitioned the dataset into 20k subsets, each containing 50 classes with 50 images per class. We use this dataset for pre-training and perform transfer learning evaluation on other unseen datasets.
\\
\par
\noindent\textbf{CIFAR-10 CIFAR-100 STL-10 Aircraft Pets.}
CIFAR-10 and CIFAR-100 are image datasets introduced by Alex Krizhevsky, containing 60000 images resized to 32$\times$32 pixels. CIFAR-10 includes 10 classes, while CIFAR-100 features 100 fine-grained classes. STL-10, derived from ImageNet, consists of 10 classes with 13000 labeled images and 100000 unlabeled images, with a resolution of 96$\times$96 pixels. The Aircraft dataset includes 10000 images across 100 aircraft models, with hierarchical labels for manufacturer, family, and variant. The Pets dataset consists of 7349 images of 37 pet breeds, with annotations for class labels, bounding boxes, and pixel-level segmentation. We use these datasets to evaluate the model's transfer learning capabilities, which means the labels of these datasets are not visible to the model.
\\
\par
\noindent\textbf{DomainNet.}
DomainNet is a dataset for multi-domain generalization. We use it to evaluate algorithms' ability for few-shot domain generalization. It consists of 345 classes from 6 domains, with a resolution of 224$\times$224 pixels. We use Clipart, Infograph, Painting, Quickdraw, and Real domains for training, while Sketch domains are for testing. The tasks we constructed are 5-way 1-shot and 20-way 5-shot. Note that the testing set shares the same 345 classes as the training set.
\\
\par
\noindent\textbf{GLUE.}
The GLUE Benchmark (General Language Understanding Evaluation) tests models on 9 diverse NLP tasks, including CoLA for grammatical acceptability, SST-2 for sentiment classification, MRPC for paraphrase detection, STS-B for sentence similarity, QQP for duplicate question detection, MNLI for natural language inference, QNLI for question-answer validation, RTE for entailment classification, and WNLI for pronoun resolution. We use this dataset to test the efficiency of different algorithms for multi-task fine-tuning on LLM models. Specifically, we use five binary classification tasks, \ie, SST-2, QQP, RTE, WNIL, and CoLA for the training of \ourmethod. Then we use the other two tasks, \ie, MRPC and QNIL, to evaluate the performance of different fine-tuning algorithms.
\\
\par

\subsection{Model Configuration}

For the downstream network $\theta$, we adopt different architectures for different tasks. In tasks on the Omniglot and Mini-ImageNet datasets, we follow the standard setup used by most meta-learning methods, employing 4 convolution blocks with batch normalization and ReLU with a linear probe for classification. For zero-shot tasks on CIFAR-10, CIFAR-100, STL-10, Aircraft, and Pets, we use the commonly adopted ResNet-12 with a linear probe for classification. In multi-domain generalization tasks, we maintain the same setup with~\cite{hierar_meta}. In the LLM multi-task fine-tuning scenario, the downstream network is the LLM model itself, \ie, RoBERTa-base. For the diffusion model $f^G_{\phi}$, \ourmethod employs the same U-Net architecture given by Meta-Diff~\cite{MetaDiff}. The task embedding $Emb_{T_i}$ is calculated by a ResNet101 backbone used by~\cite{ICIS}. In the weight preparation stage, the real-world optimizer (Adam) uses a fixed learning rate of 0.005 and an automatic early-stopping strategy~\cite{early-stop} to determine the downstream task training epoch $M$. In the meta-training stage, we set the learning rate $\eta$, meta-learning rate $\zeta$ and training epochs to 0.005, 0.001, and 6000, respectively. The segment number $k$ is set to 3, and the diffusion step $T$ is set to 20.

\clearpage
\section{Additional Experiments}
\subsection{Sensitive Study}
\begin{figure}[H]
  \centering
  \begin{minipage}[t]{0.48\linewidth}
    \centering
    \includegraphics[width=\linewidth]{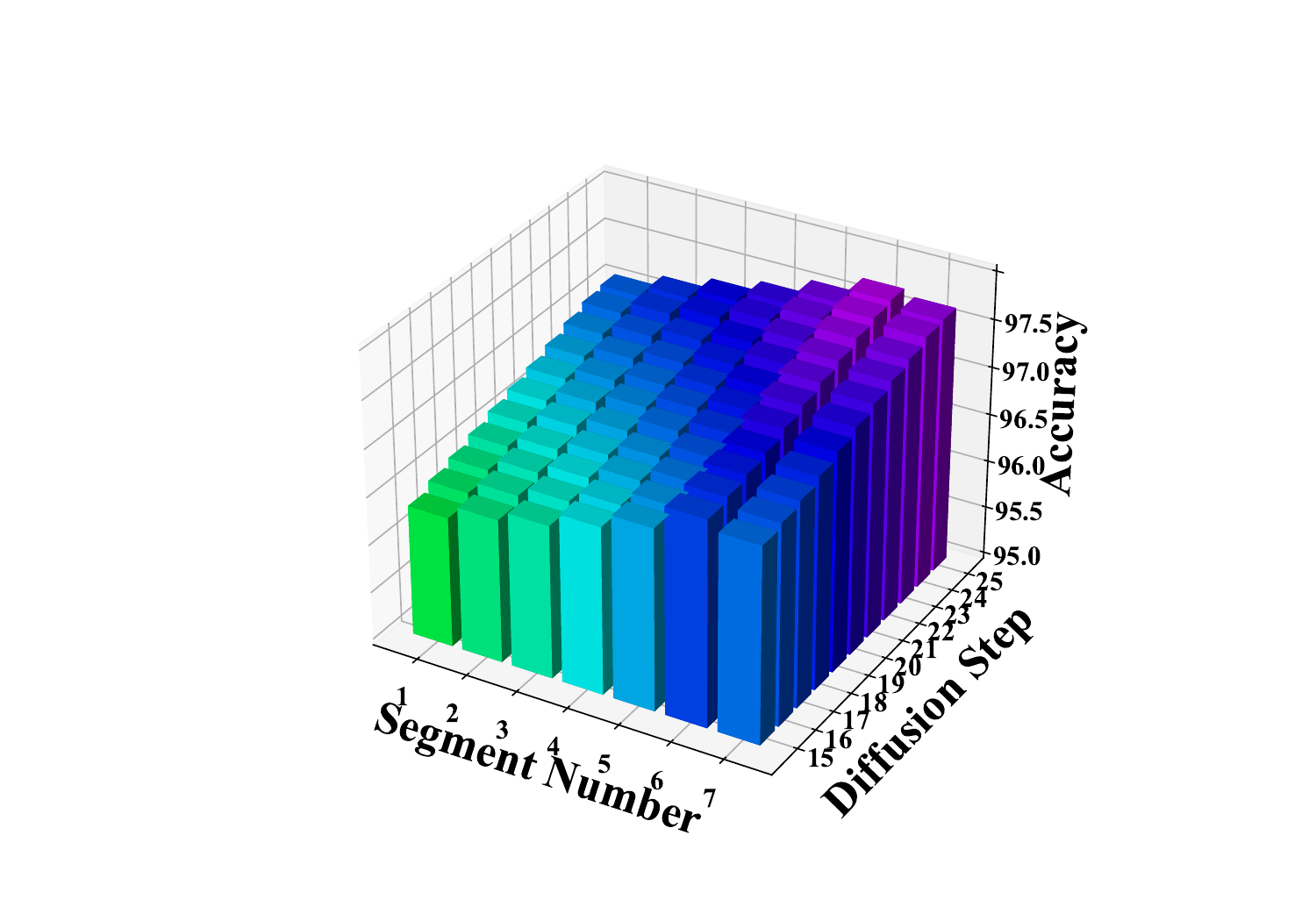}
    \caption{Sensitive study on Omniglot}
    \label{fig:sen_omni}
  \end{minipage}\hfill
  \begin{minipage}[t]{0.48\linewidth}
    \centering
    \includegraphics[width=\linewidth]{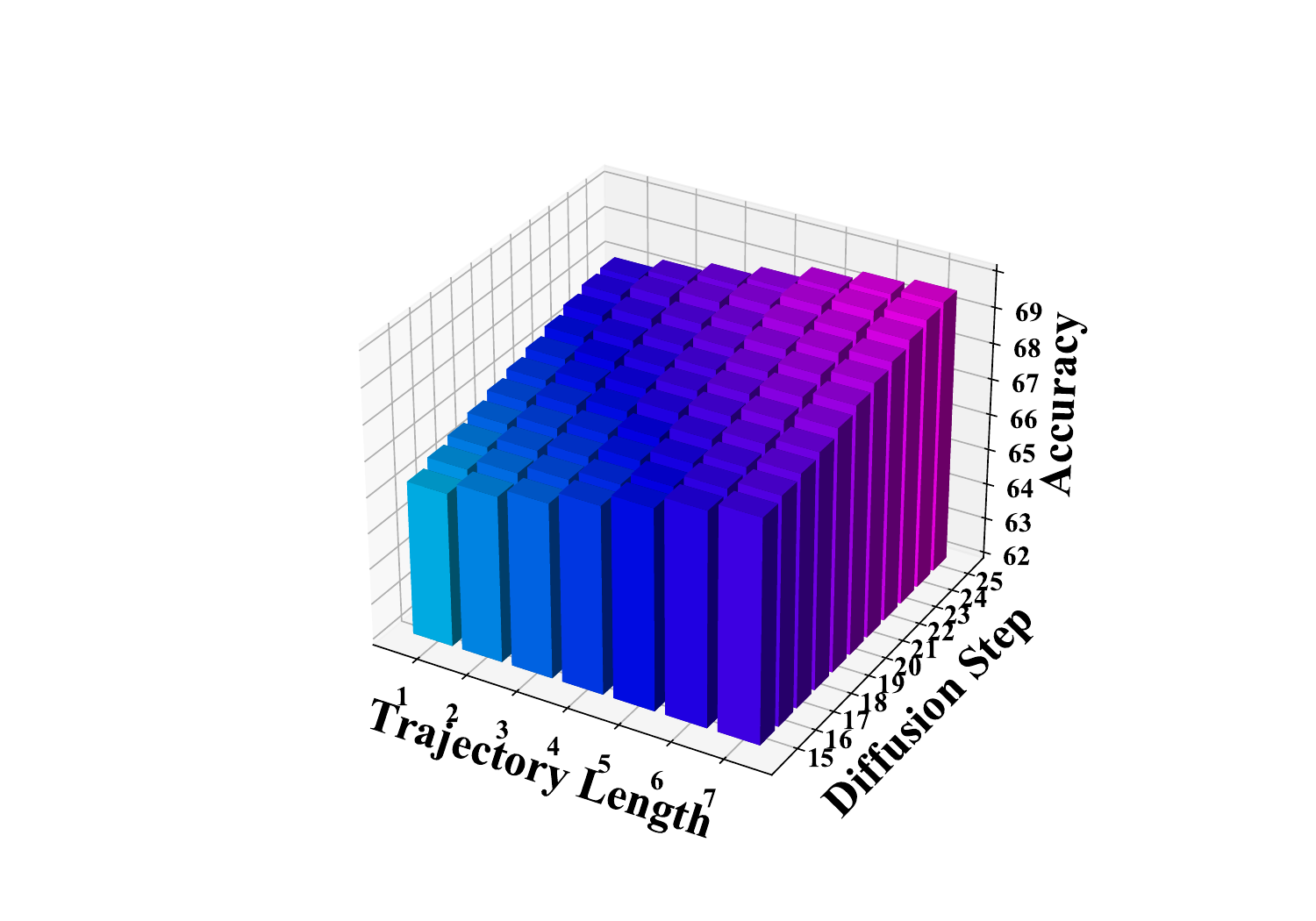}
    \caption{Sensitive study on Mini-Imagenet}
    \label{fig:sen_mini}
  \end{minipage}
\end{figure}

Compared to existing weight generation methods, the two hyperparameters of \ourmethod, \ie, segment number and diffusion steps, may affect the model's sensitivity. To evaluate the robustness of our approach, we analyze the impact of segment number and diffusion steps on classification accuracy for Mini-ImageNet and Omniglot. As shown in Figure~\ref{fig:sen_omni} and Figure~\ref{fig:sen_mini}, the accuracy remains stable across different parameter settings, with only minor variations. While increasing segment number and diffusion steps can slightly improve performance, excessive changes do not lead to significant degradation. The consistency across both datasets indicates that \ourmethod is insensitive to these hyperparameters, demonstrating robustness. Note that we do not use the parameter combination that achieves the highest accuracy in this experiment as the default setting. We aim to maintain performance above the state-of-the-art levels while reducing computational costs during inference.

\subsection{Adaptability to Different Architectures}
To verify the effectiveness of \ourmethod under different scale network structures, we compare them with their degraded versions, \ie, REPTILE, TW-Di, and \ourmethodo. Table~\ref{tab:archi_omni} and Table~\ref{tab:archi_mini} show the effectiveness of our method on Swin Transformer, ResNet18, and MobileNetV2 architectures. The results demonstrate that each component of our method remains effective across neural network architectures of different scales.

\begin{table}[h]
    \centering
    \begin{tabular}{lccc}
        \toprule
        \textbf{Method} & \textbf{Swin Transformer} & \textbf{ResNet18} & \textbf{MobileNetV2} \\
        \midrule
        REPTILE  & 98.27 & 97.11 & 93.79\\
        TW-Di  & 97.43 & 96.60 & 93.61 \\
        \ourmethodo    & 98.93 & 98.44 & 94.92 \\
        \ourmethod    & 99.90 & 98.82 & 95.84 \\
        \bottomrule
    \end{tabular}
    \caption{Comparison of accuracy across different network structures on Omniglot 5-way 1-shot tasks.}
    \label{tab:archi_omni}
\end{table}

\begin{table}[h]
    \centering
    \begin{tabular}{lccc}
        \toprule
        \textbf{Method} & \textbf{Swin Transformer} & \textbf{ResNet18} & \textbf{MobileNetV2} \\
        \midrule
        REPTILE  & 57.27  & 52.55  & 40.67  \\
        TW-Di  & 66.26  & 60.32  & 53.27  \\
        \ourmethodo    & 76.75  & 70.76  & 59.33  \\
        \ourmethod    & 82.38  & 77.85  & 64.81  \\
        \bottomrule
    \end{tabular}
    \caption{Comparison of accuracy across different network structures on Mini-ImageNet 5-way 1-shot tasks.}
    \label{tab:archi_mini}
\end{table}

\subsection{Domain Generalization}
We added the 5-way 1-shot domain generalization experiment on Meta-Dataset~\cite{meta-dataset} in Table~\ref{tab:meta_dataset_unseen}.

\begin{table*}[h]
\centering
\resizebox{\linewidth}{!}{
\begin{tabular}{lccccccc}
\toprule
\textbf{Method} & \textbf{Traffic Sign} & \textbf{MSCOCO} & \textbf{MNIST} & \textbf{CIFAR-100} & \textbf{CIFAR-10} & \textbf{Avg} & \textbf{Latency (ms)} \\
\midrule
REPTILE             & 43.32 ± 0.88 & 39.62 ± 1.02 & 68.50 ± 0.88 & 40.60 ± 0.98 & 43.73 ± 0.94 & 47.15 & 26.8 \\
Meta-Baseline       & 47.52 ± 0.92 & 41.20 ± 1.04 & 68.47 ± 0.85 & 43.07 ± 0.85 & 46.02 ± 0.90 & 49.26 & 26.2 \\
Meta-Hypernetwork   & 55.91 ± 0.91 & 48.47 ± 0.91 & 80.55 ± 0.91 & 50.67 ± 0.91 & 54.14 ± 0.91 & 57.95 & 10.3 \\
GHN3                & 44.22 ± 0.93 & 47.27 ± 1.14 & 76.44 ± 0.91 & 36.38 ± 0.83 & 42.52 ± 1.06 & 49.37 & 20.1 \\
OCD                 & 52.14 ± 0.90 & 57.74 ± 1.12 & 75.23 ± 0.92 & 47.51 ± 0.91 & 55.54 ± 1.13 & 57.63 & 9.7  \\
Meta-Diff           & 57.19 ± 0.94 & 60.01 ± 1.04 & 78.47 ± 0.94 & 55.29 ± 0.89 & 61.38 ± 1.01 & 62.47 & 10.9 \\
\ourmethodo (ours) & 58.48 ± 1.02 & 60.58 ± 1.00 & 80.82 ± 0.93 & 55.13 ± 0.93 & 62.93 ± 1.04 & 63.59 & 7.0  \\
\rowcolor{myblue}
\ourmethod (ours)  & 59.46 ± 0.99 & 62.70 ± 1.01 & 83.65 ± 0.94 & 57.07 ± 0.91 & 65.11 ± 1.05 & 65.60 & \textbf{5.2} \\
\bottomrule
\end{tabular}}
\caption{Comparison of 5-way 1-shot performance on the last 5 unseen datasets of Meta-Dataset. These methods are trained on the first 8 datasets of Meta-Dataset.}
\label{tab:meta_dataset_unseen}
\end{table*}

\clearpage
\section{Related Work}\label{sec:related_work}
\noindent\textbf{Meta-Learning.}
Meta-learning often employs a bi-level optimization-based paradigm~\cite{MAML,ANIL, meta_baseline, meta_bootstrap, on_convergence} for better generalization performance~\cite{generalization_information_1,generalization_information_2,generalization_information_3,generalization_PAC_1,generalization_stable_1} on few-shot and reinforcement learning tasks. However, it incurs significant costs for weight fine-tuning, especially in multi-task scenarios. Methods like ANIL \cite{ANIL} and Reptile \cite{REPTILE} improve training efficiency by minimizing updates to only essential task-specific layers or by approximating meta-gradients, respectively. However, these methods still rely on gradient computation and fail to achieve superior accuracy.
\vspace{5pt}
\par
\noindent\textbf{Network Weights Generation.} 
Hypernetwork~\cite{hypernetworks} is the first method that uses one network to generate another's weights, leading to extensions like the HyperSeg~\cite{hypernetwork_2} for downstream task flexibility. Conditional diffusion models provide another approach~\cite{protodiff}, with OCD leveraging overfitting~\cite{OCD} and Meta-Diff enhancing few-shot adaptability~\cite{MetaDiff}. Hyper-representations~\cite{VAE_weights} embed model characteristics to support weight generation for unseen tasks, while Image-free Classifier Injection~\cite{ICIS} achieves zero-shot classification via semantic-driven weights. These methods are constrained by their single-level optimization approach, which presents limitations in both cross-task knowledge transfer capabilities and potential adaptability to new tasks. Note that Meta-Diff essentially uses a single-level optimization strategy (see its Algorithm.1 for evidence). Most importantly, these methods overlook the supervision signal in local target weights, limiting the model's efficiency and accuracy.

% \subsection{Symbol Table}
% Potentially ambiguous symbols are described in Table~\ref{tab:symbols} to enhance the reader's comprehension of this paper.
% \begin{table}[h!]
% \centering
% \begin{tabular}{|>{\centering\arraybackslash}m{0.15\linewidth}|p{0.7\linewidth}|}
% \hline
% \textbf{Symbol} & \textbf{Description} \\ \hline
% $\epsilon$ & Gaussian noise. \\ \hline
% $\phi$ & Diffusion model, \ie, denoiser. \\ \hline
% $\{\alpha_i\}_{i=0}^T$ & Decay schedule for diffusion process. \\ \hline
% {$C_{infer}$} & Inference chain of the diffusion model. \\ \hline
% $Tra$ & Optimization trajectory. \\ \hline
% $r(\cdot)$ & Function that maps $Tra$ to $C_{\text{infer}}$ \\ \hline
% $x_i$ & Element of the inference chain. \\ \hline
% \multirow{2}{*}{$\theta^b$} & Body, \ie, bottleneck of the downstream network, also used for task embedding. \\ \hline
% $\theta^h$ & Head of the downstream network. \\ \hline
% $\hat{\theta}$ & Weight generated by the diffusion model. \\ \hline
% $\eta$ & Inner-loop learning rate. \\ \hline
% $\zeta$ & Outer-loop learning rate. \\ \hline
% $K$ & Number of Inner-loop step. \\ \hline
% \end{tabular}
% \caption{Symbols and their descriptions.}
% \label{tab:symbols}
% \end{table}

% \input{sec/rebuttal}
% \input{sec/checklist}
\end{document}